\documentclass[12pt]{article}
\usepackage{amsmath}
\usepackage{graphicx,psfrag,epsf}
\usepackage{enumerate}
\usepackage{natbib}

\newcommand{\blind}{0}

\usepackage{float}
\usepackage{subcaption}
\usepackage{caption}
\usepackage{amsfonts}
\usepackage{amssymb}
\usepackage{amsthm}
\usepackage{booktabs}
\usepackage[ruled,vlined]{algorithm2e}
\usepackage[colorlinks=false]{hyperref}
\usepackage{tabularx} 
\newcolumntype{Y}{>{\centering\arraybackslash}X}
\usepackage{titling}

\usepackage{xargs}
\usepackage[usenames,dvipsnames,svgnames,table]{xcolor}
\usepackage[colorinlistoftodos,prependcaption,textsize=tiny]{todonotes}
\newcommandx{\unsure}[2][1=]{\todo[linecolor=red,backgroundcolor=red!25,bordercolor=red,#1]{#2}}
\newcommandx{\change}[2][1=]{\todo[linecolor=blue,backgroundcolor=blue!25,bordercolor=blue,#1]{#2}}
\newcommandx{\info}[2][1=]{\todo[linecolor=OliveGreen,backgroundcolor=OliveGreen!25,bordercolor=OliveGreen,#1]{#2}}
\newcommandx{\improvement}[2][1=]{\todo[linecolor=Plum,backgroundcolor=Plum!25,bordercolor=Plum,#1]{#2}}

\newtheorem{assumption}{Assumption}
\newtheorem{theorem}{Theorem}  
\newtheorem{lemma}[theorem]{Lemma} 
\newtheorem{proposition}[theorem]{Proposition} 

\newcommand{\R}{\mathbb{R}} 

\newcommand{\ve}[2]{\langle #1 ,  #2 \rangle} 

\newcommand{\E}[1]{{\rm E}\left[#1\right] }

\DeclareMathOperator{\Exp}{\mathbf{E}} 
\DeclareMathOperator{\Cov}{Cov}         
\DeclareMathOperator{\Var}{Var}         

\DeclareMathOperator{\Vector}{vec}

\newcommand{\cA}{{\cal A}}
\newcommand{\cB}{{\cal B}}

\newcommand{\cL}{{\cal L}}

\newcommand{\cU}{{\cal U}}
\newcommand{\cV}{{\cal V}}
\newcommand{\cX}{{\cal X}}
\newcommand{\cY}{{\cal Y}}

\newcommand{\mA}{{\bf A}}
\newcommand{\mB}{{\bf B}}
\newcommand{\mC}{{\bf C}}
\newcommand{\mD}{{\bf D}}
\newcommand{\mE}{{\bf E}}

\newcommand{\mI}{{\bf I}}

\newcommand{\mL}{{\bf L}}
\newcommand{\mM}{{\bf M}}
\newcommand{\mN}{{\bf N}}

\newcommand{\mP}{{\bf P}}
\newcommand{\mQ}{{\bf Q}}
\newcommand{\mR}{{\bf R}}

\newcommand{\mT}{{\bf T}}
\newcommand{\mU}{{\bf U}}
\newcommand{\mV}{{\bf V}}

\newcommand{\mX}{{\bf X}}
\newcommand{\mY}{{\bf Y}}
\newcommand{\mZ}{{\bf Z}}

\newcommand{\mSigma}{{\bf \Sigma}}
\newcommand{\mOmega}{{\bf \Omega}}
\newcommand{\mPhi}{{\bf \Phi}}

\newcommand{\zeros}{{\bf 0}}

\begin{document}

\def\spacingset#1{\renewcommand{\baselinestretch}%
{#1}\small\normalsize} \spacingset{1}


\if0\blind
{
  \title{\bf Tensor Canonical Correlation Analysis with Convergence and Statistical Guarantees}
  \author{You-Lin Chen\\
    Department of Statistics, University of Chicago\\
    and \\
    Mladen Kolar and Ruey S. Tsay \\
    University of Chicago, Booth School of Business}
  \maketitle
} \fi

\bigskip
\begin{abstract}
  In many applications, such as classification of images or videos, it
  is of interest to develop a framework for tensor data instead of an
  ad-hoc way of transforming data to vectors due to the computational
  and under-sampling issues. In this paper, we study convergence and
  statistical properties of two-dimensional canonical correlation
  analysis \citep{Lee2007Two} under an assumption that data come from
  a probabilistic model. We show that carefully initialized the power method
  converges to the optimum and provide a finite sample bound. Then
  we extend this framework to tensor-valued data and propose the
  higher-order power method, which is commonly used in tensor
  decomposition, to extract the canonical directions. Our method can
  be used effectively in a large-scale data setting by solving the
  inner least squares problem with a stochastic gradient descent, and
  we justify convergence via the theory of Lojasiewicz's inequalities
  without any assumption on data generating process and initialization. For practical
  applications, we further develop (a) an inexact updating scheme
  which allows us to use the state-of-the-art stochastic gradient
  descent algorithm, (b) an effective initialization scheme which
  alleviates the problem of local optimum in non-convex optimization,
  and (c) a deflation procedure for extracting several canonical
  components. Empirical analyses on challenging data including gene
  expression and air pollution indexes in Taiwan, show the
  effectiveness and efficiency of the proposed methodology. Our
  results fill a missing, but crucial, part in the literature on
  tensor data.
\end{abstract}

\noindent%

\spacingset{1.45}

 \section{Introduction}

Canonical Correlation Analysis (CCA) is a widely used multivariate statistical
method that aims to learn a low dimensional representation from two sets of
data by maximizing the correlation of linear combinations of the sets of data
\citep{Hotelling1936Relations}. The two data sets could be different
views (deformation) of images, different measurements of the same object
or data with labels. CCA has been
used in many applications, ranging from text and image
retrieval \citep{Mroueh2016} and image clustering
\citep{Jin2015Cross}, to various scientific fields
\citep{Hardoon2004Canonical}.  In contrast to principal component
analysis (PCA) \citep{Pearson1901lines}, which is an unsupervised
learning technique for finding low dimensional data representation,
CCA benefits from the label information \citep{Sun2007Class},
other views of data, or correlated side information as in multi-view
learning \citep{Sun2013survey}.

In this paper, we study CCA in the context of matrix valued and tensor
valued data. Matrix valued data have been studied in
\citep{Gupta2000Matrix, Kollo2005Advanced, Werner2008estimation,
  Leng2012Sparse, Yin2012Model, Zhao2014Structured,
  Zhou2014Regularized}, while tensor valued data in
\citep{Ohlson2013multilinear, Manceur2013Maximum,
  Lu2011survey}. Practitioners often apply CCA to such data by first
converting them into vectors, which destroys the inherent structure
present in the samples, increases the dimensionality of the problem as well
as the computational complexity.
\citet{Lee2007Two} proposed a two-dimensional CCA (2DCCA) to address these problems for
matrix valued data. 2DCCA preserves data representation and alleviates
expensive computation of high dimensional data that arises from vectorizing.
However, convergence properties of the algorithm proposed in \citet{Lee2007Two} are still not known.

We present a power method that
can solve the 2DCCA problem and present a natural extension of 2DCCA to
tensor-valued data. The higher-order power
method (HOPM) and its variant are introduced for solving the resulting optimization problem.
HOPM is a common algorithm for tensor decomposition \citep{Kolda2009Tensor,
  DeLathauwer2000best, DeLathauwer2000multilinear}, with convergence
properties intensely studied in recent years \citep{Wen2012Solving,
  Uschmajew2015new, Xu2013block, Li2015convergence,
  Guan2018Convergence}.
However, these convergence results are not directly
applicable to CCA for the tensor data.
One particular reason is that the CCA constraints do not imply the identifiability.
It is also possible that iterates of HOPM get trapped in a local maximum with random initialization.
In order to establish convergence of HOMP for the tensor
valued CCA problem, a probabilistic model of 2DCCA is introduced.
Under this model, we can characterize covariance and cross-covariance matrix explicitly,
which shows HOPM provably converges to the optimum of the
objective function provided that the initialization is close enough to optimum.
For a model-free setting, we introduce a variant of HOPM with the
simple projection and show the equivalence. This insight of the equivalence
as well as the theory of analytical gradient flows and Lojasiewicz gradient
inequality provides solution to prove the convergence bound. There are two implications
from this result. First, a deflation procedure which does not
require any singular value decomposition is given for extracting multiple canonical components.
Second, the theorem guarantees existence and uniqueness of the limit of iterates of HOPM
with the arbitrary initialization. This means the HOPM converges even when initialized poorly.



\subsection{Main Contributions}

We make several contributions in this paper.

First, we interpret 2DCCA \citep{Lee2007Two} as a non-convex method for the low
rank tensor factorization problem. This allows us to formulate a tensor
extension of 2DCCA, which we term tensor canonical correlation analysis (TCCA).
We discuss convergence and statistical property under a probabilistic model. In
the literature, \citet{Kim2007Tensor} and \citet{Luo2015Tensor} discuss tensor
extension of CCA in different settings. The former stresses correlation between
each mode, while the latter discusses CCA when multiple views of data are
available. None of them provide theoretical justifications or efficient
algorithms. See Section~\ref{sec:related-work} for more discussions.

Second, we develop the higher-order power method and its variant for solving the
TCCA problem. One obstacle in analysis of tensor factorization is
unidentifiability. To circumvent this difficulty and seek further numerical
stability, we propose a variant of HOPM called Simple HOPM (sHOPM), which uses a
different normalization scheme, and show both are equivalent with same
initialization. See Proposition~\ref{prop:hopm-als} for a formal statement.
Through this insight and borrowing tools from tensor factorization, we show that
HOPM converges with arbitrary initial points.

Third, we consider several practical issues that arise in analysis of real data.
We provide an effective initialization scheme to reduce the probability of the
algorithm getting trapped in a poor local optimum and develop a faster algorithm
for large scale data. Furthermore, we discuss a deflation procedure which is
used to extract multiple canonical components without any decomposition.

Finally, we apply our method to several real data sets, including gene
expression data, air pollution data, and electricity demands. The results
demonstrate that our method is efficient and effective even in the extremely
high-dimensional setting and can be used to extract useful and interpretable
features from the data.

\subsection{Related Work}
\label{sec:related-work}

Two papers in the literature are closely related to our tensor
canonical correlation analysis. \citet{Luo2015Tensor} consider multiple-view
extension of CCA, which results in a tensor covariance
structure. However, they only consider the vector-valued data, so the
vectorization is still required in their setting. \citet{Kim2007Tensor}
extend CCA to the tensor data, focusing on 3 dimensional videos,
but only consider joint space-time linear relationships,
which is neither the low-rank approximation of CCA nor the extension of 2DCCA.
Moreover, none of the aforementioned papers
provide convergence analysis, as we detail in
Section~\ref{sec:p2DCCA} and Section~\ref{sec:convergence-analysis}.  Furthermore, 2DCCA \citep{Lee2007Two}
and other 2D extensions proposed in the literature
\citep{Yang2004Two,Ye2004Two,Zhang20052D2PCA,Yang2008Two} also lack
rigorous analysis.

Various extensions of CCA are available in the literature. Kernel
method aims to extract nonlinear relations via implicitly mapping data
to a high-dimensional feature space \citep{Hardoon2004Canonical,
  Fukumizu2007}.  Other nonlinear transformations have been
considered to extract more sophisticated relations embedded in the
data. \citet{Andrew2013Deep} utilized neural networks to learn
nonlinear transformations, \citet{Michaeli2016Nonparametric} used
Lancaster's theory to extend CCA to a nonparametric setting, while
\citet{Lopez-Paz2014Randomized} proposed a randomized nonlinear CCA,
which uses random features to approximate a kernel matrix.
\citet{Bach2006} and \citet{Safayani2018EM} interpreted CCA as a
latent variable model, which allows development of an EM algorithm for
parameter estimation. In the probabilistic setting, \citet{Wang2016b}
developed variational inference.  In general, nonlinear extensions of
CCA often lead to complicated optimization problems and lack
rigorous theoretical justifications.

Our work is also related to scalable methods for solving the CCA
problem.  Although CCA can be solved as a generalized eigenvalue
problem, computing singular value (or eigenvalue) decomposition of
a large (covariance) matrix is computationally expensive. Finding an
efficient optimization method for CCA remains an important problem,
especially for large-scale data, and has been intensively studied
\citep{Yger2012Adaptive, Ge2016Efficient, Allen-Zhu2017Doubly, Xu2018Accelerated, Fu2017Scalable, Arora2017Stochastic, Bhatia2018Gen}.  \citet{Ma2015Finding} and
\citet{Wang2016Efficienta} are particularly related to
our work.  \cite{Wang2016Efficienta} formulated CCA as a regularized
least squares problem solved by stochastic gradient descent.
\cite{Ma2015Finding} developed an augmented
approximate gradient (AppGrad) scheme to avoid computing the inverse
of a large matrix. Although AppGrad does not directly maximize the
correlation in each iteration, it can be shown that the CCA solution
is a fixed point of the AppGrad scheme. Therefore, if the
initialization is sufficiently close to the global CCA solution, the
iterates of AppGrad would converge to it.

\subsection{Organization of the paper}
In Section~\ref{sec:preliminaries}, we define basic notations and
operators of multilinear algebra and introduce CCA and the Lojasiewicz
inequality.
We present the 2DCCA and its convergence analysis in Section~\ref{sec:two-d-canon-corr-analys}.
Then we extend 2DCCA to tensor data and include a study of HOPM as well as a deflation procedure in Section~\ref{sec:TCCA}.
Section~\ref{sec:practical-considerations} deals with practical
considerations including an inexact updating rule and effective
initialization. Numerical
simulations verifying the theoretical results and applications to real
data are given in Section~\ref{sec:num-studies}.
All technical proofs can be found in the appendix.

\section{Preliminaries}
\label{sec:preliminaries}

In this section, we provide the necessary background for the subsequent
analysis, including some basic notations and operators in
multilinear algebra. The Lojasiewicz inequality and
the associated convergence results are presented in Section~\ref{sec:lojasiewicz},
and the canonical correlation analysis is briefly reviewed in Section~\ref{sec:canon-corr-analys}.

\subsection{Multilinear Algebra and Notation}
\label{sec:nutilinear-algebra}

We briefly introduce notation and concepts in multilinear algebra
needed for our analysis. We recommend \citet{Kolda2009Tensor} as the reference.
We start by introducing the notion of multi-dimensional
arrays, which are also called tensors.  The order of a tensor, also called
way or mode, is the number of dimensions the tensor has.  For example,
a vector is a first-order tensor, and a matrix is a second-order
tensor. We use the following convention to distinguish between
scalars, vectors, matrices, and higher-order tensors: scalars are
denoted by lower-case letters
($a, b ,c ,\dots ;\alpha, \beta, \gamma, \dots$), vectors by
upper-case letters ($A,B, \dots, X, Y$), matrices by bold-face
upper-case letters ($\mA, \mB, \dots, \mX, \mY$), and tensors by
calligraphic letters ($\cA, \cB, \dots, \cX, \cY$).  For an $m$-mode
tensor $\cX \in \R^{d_1 \times \dots \times d_m}$, we let its
$(i_1, \dots, i_m)$-th element be $x_{i_1 i_2 \dots i_m}$ or
$(\cX)_{i_1 i_2 \dots i_m}$.

Next, we define some useful operators in multilinear algebra.
Matricization, also known as unfolding, is a process of transforming a
tensor into a matrix. The mode-$a$ matricization of a tensor $\cX$ is
denoted by $\cX_{(a)}$ and arranges the mode-$a$ fibers to be the
columns of the resulting matrix. More specifically, tensor element
$(i_1,\dots,i_m)$ maps to matrix element $(i_a,j)$, where
$
j = 1+\sum_{k=1,k\neq a}^{m} (i_k-1) j_k$ with  $j_k = \prod_{q=1,q\neq a}^{k-1} i_q$.
The vector obtained by vectorization of $\cX$ is denoted by
$\mbox{vec}(\cX)$.  The Frobenius norm of the tensor $\cX$ is defined
as $\| \cX \|^2_F = \langle \cX, \cX \rangle$,
where $ \langle \cdot, \cdot \rangle$ is the inner product defined on
two tensors $\cX \in \R^{d_1 \times \dots \times d_m}$,
$\cY \in \R^{d_1 \times \dots \times d_m}$ and given by
\[
\langle \cX, \cY \rangle = \sum_{i_1=1}^{d_1} \cdots \sum_{i_m=1}^{d_m} x_{i_1 i_2 \dots i_m} y_{i_1 i_2 \dots i_m}.
\]
The mode-$k$ product of a tensor
$\cX \in \R^{d_1 \times \dots \times d_m}$ with a matrix
$\mA \in \R^{a \times d_k}$ is a tensor of size
$d_1 \times \dots \times d_{k-1} \times a \times d_{k+1} \dots \times
d_m $ defined by
\[
(\cX \times_k \mA)_{i_1 \dots i_{n-1}  j i_{n+1} \dots i_m} = \sum_{i_k=1}^{d_k} x_{i_1 i_2 \dots i_m} a_{j i_k}.
\]
The outer product of vectors $U_1 \in \R^{d_1}$, \dots,
$U_m \in \R^{d_m}$ is an $m$-order tensor defined by
\[
(U_1 \circ \dots \circ U_m)_{i_1 i_2 \dots i_m} = (U_1)_{i_1} \dots (U_m)_{i_m}.
\]
The Kronecker product of two matrices $\mA \in \R^{m \times n}$ and
$\mB \in \R^{p \times q}$ is an $mp \times nq$ matrix given by
$\mA \otimes \mB = (a_{ij} \mB)_{mp \times nq}$.
We call $\cX$  a rank-one tensor if there exist vectors $X_1, \dots, X_d$ such that
$\cX = X_1 \circ \cdots \circ X_{d}$. Given $n$ samples
$\{\cX_t\}_{t=1}^n$, it is also useful to define the
data tensor $\cX_{1:n} \in \R^{n\times d_1 \times \dots \times d_m}$
with elements $(\cX_{1:n})_{t,j_1,\dots,j_m}  = (\cX_{t})_{j_1,\dots,j_m}$.

\subsection{Convergence Analysis via Lojasiewicz Inequality}
\label{sec:lojasiewicz}

There are two common approaches for establishing convergence of
non-convex methods. One assumes that the initial points lie in a convergence region
in which the optimization procedure converges to the global optimum. For
certain problems, there are effective initialization schemes that
provide, with high probability, an initial point that is close enough
(within the convergence region) to the global optimum. We provide this
type of analysis in Section~\ref{sec:2DCCA-conv-analysis} and an
effective initialization scheme in Section~\ref{sec:effetive-initial}.

Another approach is based on the theory of analytical gradient flows,
which allows establishing convergence of gradient descent
based algorithms for difficult problems, such as, non-smooth,
non-convex, or manifold optimization. This approach is also useful for
problems arising in tensor decomposition where the optimum set is not compact and isolated.
For example, in the case of PCA or CCA problems,  the optimum set is a low
dimensional subspace or manifold. The advantage of this approach is that it can
be applied to study convergence to all stationary points, without
assuming a model for the data generating process.
The drawback of the approach is that the convergence rate
on a particular problem depends on the Lojasiewicz exponent in
Lemma~\ref{lemma:Lojineq} below, which is hard to explicitly compute.
However, \citet{Liu2018Quadratic} recently showed
that the Lojasiewicz exponent at any critical point of the quadratic
optimization problem with orthogonality constraint is $1/2$, which
leads to linear convergence of gradient descent. \citet{Li2018Calculus}
developed various calculus rules to deduce the Lojasiewicz
exponent.

The method of Lojasiewicz gradient inequality allows us to study
an optimization problem
\begin{equation} \label{optim}
\min_{Z \in \R^p} f(Z),
\end{equation}
where $f$ may not be convex, and we apply a gradient based algorithm for finding
stationary points. The key ingredient for establishing linear or
sublinear convergence rate is the following Lojasiewicz gradient
inequality \citep{Lojasiewicz1965}.

\begin{lemma}[Lojasiewicz gradient inequality]
    \label{lemma:Lojineq}
    Let $f$ be a real analytic function on a neighborhood of a stationary point $X$ in $\R^n$. Then, there exist constants $c>0$ and the Lojasiewicz
    exponent $\theta \in (0,1/2]$, such that
	\begin{equation} \tag{L} \label{Lineq}
	|f(Y)-f(X)|^{1-\theta} \leq c \| \nabla f(Y) \|,
	\end{equation}
where $Y$ is in some neighborhood of $X$.
\end{lemma}

See \citet{Absil2005Convergence} and references therein for the proof
and discussion. Since the objective function appearing in the CCA
optimization problem is a
polynomial, we can use the Lojasiewicz gradient inequality
in our convergence analysis. Suppose $\{Z_k\}_{k=1}^\infty$ is a sequence of iterates
produced by a descent algorithm that satisfy
the following assumptions:
\begin{description}
	\item[$\bullet$ Primary descent condition:] there exists $\sigma>0$ such that, for a
	sufficiently large $k$, it holds that
	\begin{equation} \tag{A1} \label{A1}
	f(Z_k)-f(Z_{k+1}) \geq \sigma \| \nabla f(Z_k) \| \| Z_k-Z_{k+1} \|.
	\end{equation}

	\item[$\bullet$ Stationary condition:] for a sufficiently large $k$, it holds that
	\begin{equation} \tag{A2} \label{A2}
	\nabla f(Z_k) =0 \quad \Longrightarrow \quad Z_k = Z_{k+1}.
	\end{equation}

	\item[$\bullet$ Asymptotic small step-size safeguard:] there exists $\kappa>0$ such that,
	for a large enough $k$, it holds that
	\begin{equation} \tag{A3} \label{A3}
	\| Z_{k+1} - Z_k \| \geq \kappa \| \nabla f(Z_k) \|.
	\end{equation}
\end{description}
Then, we have the following theorem which is the main tool we use in our analysis and
its proof can be found in \citet{Schneider2015Convergence}.
\begin{lemma}[\cite{Schneider2015Convergence}] \label{convthm} Under the condition of Lemma
  \ref{lemma:Lojineq} and assumptions \eqref{A1} and \eqref{A2}, if
  there exists a cluster point $Z^\ast$ of the sequence $(Z_k)$
  satisfying \eqref{Lineq}, then $Z^\ast$ is the limit of the
  sequence, i.e., $Z_k \rightarrow Z^\ast$. Furthermore, if \eqref{A3}
  holds, then
	\begin{equation*}
	\| Z_k - Z^\ast \|  \leq C \begin{cases}
	e^{-c k} & \text{if $\theta=\frac{1}{2}$ for some $c>0$,}\\
	k^{-\theta}  & \text{if $0 < \theta < \frac{1}{2}$,}
	\end{cases}
	\end{equation*}
	for some $C > 0$. Moreover, $\nabla f(Z_k)  \rightarrow 0$.
\end{lemma}

It is not hard to see that the gradient descent iterates for a
strongly convex and smooth function $f$ satisfy conditions
\eqref{A1}-\eqref{A3} and Lojasiewicz gradient inequality.
Moreover, if one can show that Lojasiewicz's inequality holds with
$\theta = 1/2$ for the set of stationary points,
then linear convergence to stationary points can be established
\citep{Liu2018Quadratic}. In fact, they related Lojasiewicz's inequality
to the following error bound \citep{Luo1993Error}:
\[
    \| f(X) \| \geq\mu \| X - X^\ast \|,
\]
where $\mu>0$ and $X^\ast$ is in the set of stationary points.
\citet{Karimi2016Linear} proved that these two conditions actually imply
quadratic growth:
\[
    f(X) - f(X^\ast) \geq \frac{\mu'}{2} \| X - X^\ast \|^2,
\]
where $\mu'>0$.
This means that $f$ is strongly convex when $X$ is close to a stationary point.
Those local conditions are useful in non-convex and even non-smooth,
non-Euclidean setting. See \citet{Karimi2016Linear} and
\citet{Bolte2017error} for additional discussion.

\subsection{Canonical Correlation Analysis}
\label{sec:canon-corr-analys}

Unless otherwise stated, we assume that all random elements in this paper are
zero-mean.
We first review the classical canonical correlation analysis in this
section.  Consider two multivariate random vectors
$X \in \R^{d_x}$ and $Y \in \R^{d_y}$. The canonical
correlation analysis aims to maximize the correlation between the
projections of two random vectors and can be formulated as the
following maximization problem
\begin{equation}
\label{eq:cca}
\max_{U \in \R^{d_x}, V \in \R^{d_y}} \text{corr}(U^\top X,  V^\top Y)
=
\max_{U \in \R^{d_x}, V \in \R^{d_y}} \frac{\text{cov}(U^\top X,  V^\top Y) }{\sqrt{\text{var}(U^\top X) \text{var}(V^\top Y)}}.
\end{equation}
The above optimization problem can be written equivalently
in the following constrained form
\[
\max_{U, V} U^\top \mSigma_{XY}  V\qquad \text{ s.t. }\qquad U^\top
\mSigma_{XX}  U = 1 = V^\top \mSigma_{YY} V,
\]
where, under the zero-mean assumption, $\mSigma_{XY} = \Exp[ XY^\top ]$, $\mSigma_{XX}= \Exp[XX^\top]$,
and $\mSigma_{YY}= \Exp[YY^\top]$.  Using the standard technique of
Lagrangian multipliers, we can obtain $U$ and $V$ by solving the
following generalized eigenvalue problem
\begin{equation} \label{gepCCA}
\begin{bmatrix} 0  & \mSigma_{XY} \\ \mSigma_{YX}  & 0 \end{bmatrix}
\left[ \begin{array}{c} U \\ V \end{array} \right]
= \lambda
\begin{bmatrix} \mSigma_{XX} & 0 \\ 0  & \mSigma_{YY} \end{bmatrix}
\left[ \begin{array}{c} U \\ V \end{array} \right],
\end{equation}
where $\mSigma_{YX}=\mSigma_{XY}^\top$ and $\lambda$ is the Lagrangian
multiplier. In practice, the data distribution is unknown, but
i.i.d.~draws from the distribution are available.  We can estimate the
canonical correlation coefficients by replacing the expectations with
sample averages, leading to the empirical version of the optimization
problem.

\section{Two-Dimensional Canonical Correlation Analysis}
\label{sec:two-d-canon-corr-analys}

\cite{Lee2007Two} proposed two-dimensional canonical correlation
analysis (2DCCA), which extends CCA to matrix-valued data. The naive
way of dealing with matrix-valued data is to reshape each data matrix
into a vector and then apply CCA. However, this naive
preprocessing procedure breaks the
structure of the data and introduces several side effects, including
increased computational
complexity and larger amount of data required for accurate estimation of
canonical directions. To overcome this difficulty, 2DCCA maintains the
original data representation and performs CCA-style dimension
reduction as follows. Given two data matrices $\mX \in \R^{d_1 \times d_2} , \mY \in \R^{d_1 \times d_2}$, 2DCCA
seeks left and right transformations, $L_x, R_x, L_y, R_y$, that
maximize the correlation between $L_x^\top \mX R_x$ and
$L_y^\top \mY R_y$ and can be formulated as
\begin{equation}
    \label{eq:2DCCA}
    \max_{L_x, L_y \in \R^{d_1}, R_x, R_y \in \R^{d_2}} \Cov(L_x^\top \mX R_x, L_y^\top \mY R_y)
    \qquad \text{ s.t. } \qquad \Var(L_x^\top \mX R_x) = 1 = \Var(L_y^\top \mY R_y).
\end{equation}
where we assume $\mX, \mY$ have the same shape for simplicity. 

To solve the problem above, for fixed $R_x$ and
$R_y$, $\mX R_x$ and $\mY R_y$ are random vectors and $L_x$ and $L_y$
can be obtained using any CCA algorithm.
Similarly we can switch the roles of variables and fix $L_x$ and $L_y$.
Thus, the algorithm iterates between updating $L_x$ and $L_y$ with $R_x$ and $R_y$ fixed
and updating $R_x$ and $R_y$ with $L_x$ and $L_y$ fixed.

To the best of our knowledge, no clear interpretation on
why 2DCCA
can improve the classification results, how the method is
related to data structure, nor rigorous convergence analysis
exists for 2DCCA. We address these issues,
by first noting that the 2DCCA problem in \eqref{eq:2DCCA} can
be expressed as
\begin{equation}
  \label{eq:2dcca:low_rank_restriction}
  \max_{\mU, \mV \in \R^{d_1\times d_2}: \mbox{rank}(\mU) = 1 = \mbox{rank}(\mV)} \text{corr}( \operatorname{Trace}(\mU^\top \mX),  \operatorname{Trace}(\mV^\top \mY)),
\end{equation}
where $\operatorname{Trace}(\dot)$ denotes the trace of a matrix.
Therefore, we can treat 2DCCA as the CCA optimization problem with the
low-rank restriction on canonical directions. Furthermore, formulating the optimization problem as in
\eqref{eq:2dcca:low_rank_restriction} allows us to use techniques
based on the Burer-Monteiro factorization \citep{Chen2015Fast,
  Park2016Finding, Park2017Non, Sun2017decomposition,
  Haeffele2014Structured} and non-convex optimization
\citep{Yu2018Provable, Zhao2015Nonconvex, Yu2017Influence,
  Li2019Nonconvex, Yu2018Learning}.




\subsection{Probabilistic model of 2DCCA} \label{sec:p2DCCA}

We now introduce a probabilistic
model of 2DCCA, which is closely related to factor models
appearing in the literature
\citep{Virta2017Independent,Virta2018JADE,
  Jendoubi2019whitening}. Consider two random matrices,
$\mX, \mY$,
distributed according to the following probabilistic model
\begin{equation}
\label{eq:p2DCCA}
\begin{split}
    \mX = {\bf \Phi}_1  (\mZ + \mE_{x}) {\bf \Phi}_2^\top \in \R^{d_1 \times d_2}, \ \ \ \mY = {\bf \Omega}_1 (\mZ + \mE_{y}) {\bf \Omega}_2^\top \in \R^{d_1 \times d_2},
\end{split}
\end{equation}
where $\mZ$, $\mE_{x}$, $\mE_{x}$ are random matrices of appropriate dimensions,
satisfying $\Var(\Vector{(\mZ + \mE_{x})})=\mI=\Var(\Vector{(\mZ + \mE_{y})})$, and
${\bf \Phi}_1, {\bf \Omega}_1 \in \R^{d_1 \times k}, {\bf \Phi}_2, {\bf \Omega}_2\in \R^{d_2 \times k}$
are fixed matrices. Note that this probabilistic 2DCCA model is
different from the one in \citet{Safayani2018EM} due to different noise structure.
We refer to the 2DCCA model in \citet{Safayani2018EM} as a 2D-factor model
and to our model in \eqref{eq:p2DCCA} as the probabilistic model of 2DCCA or simply, a p2DCCA model.

In practice, the noise structure appearing in the p2DCCA model
is more  natural than that in the 2D-factor model.
Consider the example of image data where  $\mX$ and $\mY$
are images of the same object with different illumination conditions.
The 2D-factor model implies additive error on images, which is unrealistic.
In contrast, the p2DCCA model considers latent variable $\mZ$,
which corresponds to the common representations, while $\mX$ and $\mY$ are different views
obtained from the transformation defined by $\mPhi_1, \mPhi_2, \mOmega_1, \mOmega_2$.
$\mZ$ can represent the magnitude of illumination, the angle of a face,
or the gender of a person. The p2DCCA model assumes the error is on the common
representation, e.g., the magnitude of illumination. This exactly fits
the description of our example and the explanations of latent models.
Hence, it is reasonable to assume that $\mZ$'s elements are independent or uncorrelated,
which admits a causal interpretation for the latent variables that data satisfy 
the principle of independent mechanisms in causal inference \cite{peters2017elements}. 
That is, the causal generative process of a hidden variables is composed of autonomous 
modules that do not influence each other. In contrast, the factor model can be better 
suited for the situations where the additive error is a natural assumption, 
e.g., time series data \citep{Wang2016, Chen2017, Chen2017Constrained}.


Note that we can always apply CCA to the data after converting them to vectors.
However, this leads to $d_1 d_2$-by-$d_1 d_2$ matrix inversion, while by exploiting the
structure in 2DCCA we only need to solve $d_1$-by-$d_1$ and $d_2$-by-$d_2$ matrix
inversion, which decreases requirement for memory and computation. One
drawback of the 2DCCA method is that the problem is non-convex and it is
possible to get stuck in local optimums even if we have infinite samples, that
is, at the population level. We can see why this phenomenon exists in the next
section.

\subsection{Convergence Analysis}\label{sec:2DCCA-conv-analysis}

In this section, we discuss the convergence behavior of power method for 2DCCA. 
Suppose that 
\begin{equation*}
    \Var(\Vector{(\mZ + \mE_{x})})=\mI=\Var(\Vector{(\mZ + \mE_{y})}, \ \ \
    \Var(\Vector{(\mZ})) = \operatorname{diag}{(\theta_{11}, \theta_{21}, \dots, \theta_{kk})}
\end{equation*}
where $\theta_{11} > \theta_{21} \geq \dots \geq \theta_{kk}$.
While we assume that the elements of $\mZ$ are uncorrelated and 
real data may not be generated from this model, this 
probabilistic model can approximately describe data
with some ``low rank structure." For simplicity, 
we assume the approximate error is zero, but it is possible that
the analysis can be extended to the approximate case.

The solution to 1DCCA problem can be obtained as follows.
Given covariances and a cross-covaraince $\mSigma_{XX}, \mSigma_{YY}, \mSigma_{XY}$,
if $\bar{U}^\ast, \bar{V}^\ast$ are the first left and right singular vectors of $\mSigma_{XX}^{-1/2} \mSigma_{XY}
\mSigma_{YY}^{-1/2}$, respectively, 
then the first 1DCCA components are $U^\ast=\mSigma_{XX}^{-1/2}\bar{U}^\ast$ and $V^\ast=\mSigma_{YY}^{-1/2}\bar{V}^\ast$.
The following proposition characterizes optimums of 1DCCA and 2DCCA.
The proof is deferred to Appendix~\ref{sec:prof-prop:optimum-1DCCA-2DCCA}.

\begin{proposition} \label{prop:optimum-1DCCA-2DCCA}
    $U^\ast$ and $V^\ast$ admit a low-rank structure. That is, there exist $U_1^\ast, U_2^\ast, V_1^\ast, V_2^\ast$ such that
    \begin{equation} \label{eq:optimum-2DCCA}
        U^\ast = U_2^\ast \otimes U_1^\ast, \ \ \ V^\ast = V_2^\ast \otimes V_1^\ast.
    \end{equation}
    Moreover, the optimums of 2DCCA \eqref{eq:2DCCA} and 1DCCA \eqref{eq:cca} coincide.
\end{proposition}

Thus, we have the closed form of global optimum of 2DCCA under the p2DCCA model.
Under the model assumption \eqref{eq:p2DCCA},
for $(i,j)=(1,2),(2,1)$, letting
\begin{equation}
\mSigma_{XX,j} = \E{\mX_j U_i U_i^\top \mX_j^\top}, \
\mSigma_{YY,j} = \E{\mY_j V_i V_i^\top \mY_j^\top}, \
\mSigma_{XY,j} = \E{\mX_j U_i V_i^\top \mY_j^\top},
\end{equation}
where
\begin{equation}\label{eq:modeXY}
\mX_j =
\begin{cases}
\mX & \text{ if } j=1\\
\mX^\top & \text{ if } j=2
\end{cases}, \quad
\mY_j =
\begin{cases}
\mY & \text{ if } j=1\\
\mY^\top & \text{ if } j=2
\end{cases},
\end{equation}
we show the power method of 2DCCA converges to the global optimum
in the population level provided that the initial point is close enough to the global optimum.

\begin{theorem} \label{thm:p-conv-p2DCCA}
	Consider the following power method iterations for the
	left and right loading vectors in 2DCCA:
	\begin{equation} \label{eq:HOPM-sim}
	\begin{split}
	\tilde{U}_{k+1,i} &= \mSigma_{XX,j}^{-1} \mSigma_{XY,j}  V_{k,i}, \ \ \ U_{k+1,i} = \tilde{U}_{k+1,i} / \sqrt{ \tilde{U}_{k+1,i}^\top \mSigma_{XX,j} \tilde{U}_{k+1,i} }, \\
	\tilde{V}_{k+1,i} &= \mSigma_{YY,j}^{-1} \mSigma_{XY,j}^\top  U_{k,i}, \ \ \ V_{k+1,i} = \tilde{V}_{k+1,i} / \sqrt{\tilde{V}_{k+1,i}^\top \mSigma_{YY,j} \tilde{V}_{k+1,i}},
	\end{split}
	\end{equation}
	Define $\cos_{\square, j} \theta(W_1, W_2) = W_1^\top \mSigma_{\square\square,j} W_2 / \sqrt{ W_1^\top \mSigma_{\square\square,j} W_1 W_2^\top \mSigma_{\square\square,j} W_2}$ 
	and $\sin_{\square, j}^2 \theta(W_1, W_2) = 1-\cos_{\square, j}^2 \theta(W_1, W_2)$ for $\square = X, Y$.
	Suppose initialization vectors
	$U_{0,i}$ and $V_{0,i}$ satisfy, for $(i,j)=(1,2), (2, 1)$
	\begin{align*}
	\theta_{12} < \theta_{11} |\cos_{X, j} \theta( U_{i}^\ast, U_{0,i}) \cos_{Y, j} \theta(V_{i}^\ast, V_{0,i})|- \theta_{12} |\sin \theta_{X, j}(U_{i}^\ast, U_{0,i}) \sin_{Y, j} \theta(V_{i}^\ast, V_{0,i})| := \tilde{\theta}_{11}.
	\end{align*}
	Then
	\[
	\max_{(i,j)} \left\{ |\sin_{X,j} \theta(U_{i}^\ast, U_{T,i})|,  |\sin_{Y,j} \theta(V_{i}^\ast, V_{T,i}))| \right\} \leq \epsilon,
	\]
	for $T \geq c(1-\theta_{12}/\tilde{\theta}_{11})^{-1} \log(1/\epsilon)$,
	where $c$ is a large enough constant depending on the initial values.
\end{theorem}

The proof is in Appendix~\ref{sec:prof-thm-p-conv-p2DCCA}. From the proof,
we can see that without careful initialization, the
power method may converge to arbitrary points. Our analysis can be extended to $(r_1,r_2)$-2DCCA 
and higher order cases where $(r_1,r_2)$ is the number of components of left and right loading matrices.

Next, we establish finite-sample bound.
Suppose we have  i.i.d.~data $\{\mX_t, \mY_t\}_{t=1}^n$
sampled from the data generating process \eqref{eq:p2DCCA}.  For $(i,j)=(1,2),(2,1)$, let
\begin{equation}
\begin{split}
\hat{\mSigma}_{XX,j} = \frac{1}{n} \sum_{t=1}^n \mX_{j,t} U_i U_i^\top \mX_j^\top,\
\hat{\mSigma}_{YY,j} = \frac{1}{n} \sum_{t=1}^n\mY_{j,t} V_i V_i^\top \mY_{j,t}^\top,\
\hat{\mSigma}_{XY,j} = \frac{1}{n} \sum_{t=1}^n \mX_{j,t} U_i V_i^\top \mY_{j,t}^\top,
\end{split}
\end{equation}
where $\mX_{j,t}, \mY_{j,t}$ are defined in \eqref{eq:modeXY}.
The regularization is necessary for finite sample estimation,
since data have low rank structure and CCA involves matrix
inversion that can amplify noise arbitrarily.
We state the sample version of the convergence theorem next.

\begin{theorem} \label{thm:p-conv-p2DCCA-sample}
	Suppose that $\max\{\|\mX\|, \|\mX^\top\|, \|\mY\|, \|\mY^\top\|\}\leq 1$ almost surely.
	Consider the following power method iterations for the
	left and right loadings in 2DCCA:
	\begin{equation*} 
	\begin{split}
	    \tilde{U}_{k+1,i} =& (\mSigma_{XX,j}+\epsilon_n \mI)^{-1} \mSigma_{XY,j} V_{k,i}, \ \ \ U_{k+1,i} = \tilde{U}_{k+1,i}/\sqrt{\tilde{U}_{k+1,i}^\top (\mSigma_{XX,j}+\epsilon_n \mI) \tilde{U}_{k+1,i}} \\
	    \tilde{V}_{k+1,i} =& (\mSigma_{YY,j}+\epsilon_n \mI)^{-1} \mSigma_{XY,j}^\top U_{k,i}, \ \ \ V_{k+1,i} = \tilde{V}_{k+1,i}/\sqrt{\tilde{V}_{k+1,i}^\top (\mSigma_{YY,j}+\epsilon_n \mI) \tilde{V}_{k+1,i}}.
	\end{split}
	\end{equation*}
	for $(i,j)=(1,2), (2,1)$. Define  
	\[
    \widehat{\cos}_{\square, j} \theta(W_1, W_2) = W_1^\top (\hat{\mSigma}_{\square\square,j}+\epsilon_n \mI) W_2 / \sqrt{ W_1^\top (\hat{\mSigma}_{\square\square,j}+\epsilon_n \mI) W_1 W_2^\top (\hat{\mSigma}_{\square\square,j}+\epsilon_n \mI)W_2} 
    \] 
    and $\widehat{\sin}_{\square, j}^2 = 1- \widehat{\cos}_{\square, j}^2$ for $\square = X,Y$.
	Suppose initialization satisfies $\bar{U}_{0,i}, \bar{V}_{0,i}$ such that for all $i=1,2$
	\begin{align*}
	\theta_{12} < \theta_{11} |\cos \theta(\bar{U}_{i}^\ast, \bar{U}_{0,i}) \cos \theta(\bar{V}_{i}^\ast, \bar{V}_{0,i})|- \theta_{12} |\sin \theta(\bar{U}_{i}^\ast, \bar{U}_{0,i}) \sin \theta(\bar{V}_{i}^\ast, \bar{V}_{0,i})| := \tilde{\theta}_{11}.
	\end{align*}
	Then, with probability $1-\max \{d_1,d_2\} \exp{(- c_1 /8)}$, we have
	\begin{equation*}
    \max_{(i,j)} \left\{ |\widehat{\sin}_{X,j} \theta(U_{i}^\ast, U_{T,i})|,  |\widehat{\sin}_{Y, j} \theta(V_{i}^\ast, V _{T,i}))| \right\} \leq \epsilon,
    \end{equation*}
	for $T \geq c_2 (1-\theta_{12}/\tilde{\theta}_{11})^{-1} \log(1/\epsilon)$,
	provided that $\epsilon_n + \epsilon_n^{-3/2} n^{1/2} \leq c$,
	where $c_1, c_2$ are constants depending on initialization and
	$c$ is a small constant depending on initialization,
	$c_1$, $\varepsilon$,$\theta_{11}, \theta_{12}$, and $\mPhi_1, \mPhi_2, \mOmega_1, \mOmega_2$.
	In particular, if $\epsilon_n + \epsilon_n^{-3/2} n^{-1/2} \rightarrow 0$, we obtain consistency.
\end{theorem}

The proof is detailed in Appendix~\ref{sec:p-conv-p2DCCA-sample}. The key
ingredient we use in the proof is the robustness analysis of power method. Using
the concentration bound of bounded random matrices, with careful initialization,
we can also prove the power method converges to optimum with high probability as long as the sample size is large enough.

\section{Tensor Canonical Correlation Analysis}
\label{sec:TCCA}

In this section, we extend CCA to the tensor-valued data and the higher-order power method and its variant.
We also establish convergence results without any model assumption.
The last part discusses how to find multiple canonical variables.

\subsection{Canonical Correlation Analysis with Tensor-valued Data}
\label{sec:tensor-canon-corr-analys}

We are ready to present the problem of tensor canonical correlation
analysis, which is a natural extension of 2DCCA. Consider two
zero-mean random tensors
$\cX \in \R^{d_{1} \times \dots \times d_{m}}$ and
$\cY \in \R^{d_{1} \times \dots \times d_{m}}$. Note that, for
simplicity of presentation, we assume the two random tensors $\cX$
and $\cY$ have the same mode and shape. This assumption is
not needed in the analysis or the algorithm. TCCA seeks two rank-one
tensors
$\cU = U_1 \circ \cdots \circ U_{m} \in \R^{d_{1} \times \dots \times
  d_{m}}$ and
$\cV = V_1 \circ \cdots \circ V_{m} \in \R^{d_{1} \times \dots \times
  d_{m}}$ that maximize the correlation between
$\langle \cU, \cX \rangle$ and $\langle \cV, \cY \rangle$,
\begin{equation}
    \label{eq:tcca:population}
    \max_{\cU, \cV} \mbox{Corr}(\langle \cU, \cX \rangle, \langle \cV, \cY \rangle).
\end{equation}
Since the population distribution is unknown, by replacing covariance
and cross-covariance matrices by their empirical counterparts, we get
the empirical counterpart of the optimization problem in
\eqref{eq:tcca:population}
\[
  \max_{\cU, \cV} \rho_n(\cU, \cV),
\]
where $\rho_n(\cU, \cV)$ is the sample correlation defined as
\begin{equation}
    \label{eq:def:corr}
    \rho_n(\cU, \cV) = \frac{\frac{1}{n} \sum_{t=1}^{n}\ve{\cU}{\cX_t} \ve{\cV}{\cY_t}}{\sqrt{\frac{1}{n} \sum_{t=1}^{n}\ve{\cU}{\cX_t}^2 \cdot  \frac{1}{n} \sum_{t=1}^{n}\ve{\cV}{\cY_t}^2}},
\end{equation}
and $\{\cX_t, \cY_t\}_{t=1}^n$ are the samples from the unknown
distributions of $\cX, \cY$. The following empirical version of
residual form of the problem \eqref{eq:tcca:population} will be useful for
developing efficient algorithms
\begin{equation}
  \label{eq:tcca:sample2}
  \min_{\cU, \cV} \frac{1}{2n} \sum_{t=1}^{n} (\ve{\cU}{\cX_t} - \ve{\cV}{\cY_t})^2
  \qquad\text{ s.t. }\qquad
  \frac{1}{n} \sum_{t=1}^{n}\ve{\cU}{\cX_t}^2 = 1 = \frac{1}{n} \sum_{t=1}^{n}\ve{\cV}{\cY_t}^2.
\end{equation}
This formulation reveals that TCCA is related to tensor decomposition.
Furthermore, the sub-problem obtained by fixing all components of $\cU, \cV$ except for components
$U_j$ or $V_j$ is equivalent to the least squares problem.
This allows us to use the state-of-the-art solvers based on (stochastic) gradient descent that
are especially suitable for large-scale data
\citep{Wang2016Efficienta} and develop computationally efficient
algorithms for CCA with tensor data proposed in
Section~\ref{sec:practical-considerations}.

\subsection{Higher-order Power Method}
\label{sec:hopm}

The power method is a practical tool for finding the leading eigenvector of a
matrix, which is used in tensor factorization. There are many interpretations
for the power method and here we show that HOPM solve the stationary condition
of a certain Lagrange function.


The Lagrange function associated with the optimization problem in
\eqref{eq:tcca:sample2} is
\begin{equation}
  \label{eq:lagrangeTCCA}
  \cL(\cU, \cV, \lambda, \mu) = \frac{1}{2 n} \sum_{t=1}^{n} \left( \ve{\cU}{\cX_t} - \ve{\cV}{\cY_t} \right)^2 + \lambda \left( 1- \frac{1}{n} \sum_{t=1}^{n}\ve{\cU}{\cX_t}^2 \right) +\mu \left( 1 - \frac{1}{n} \sum_{t=1}^{n}\ve{\cV}{\cY_t}^2 \right),
\end{equation}
where $\lambda$ and $\mu$ are the Lagrange multipliers. Minimizing the
above problem in one component $U_j$ of $\cU$, with the other components
of $\cU,\cV$ fixed, leads to a least squares
problem. Define the partial contraction $\mX_{j}$ of $\cX_{1:n}$
with all component of $\cU$ except $U_{j}$ as
\[
\mX_{j} = \cX_{1:n} \times_2 U_{1}^\top \cdots \times_{j} U_{j-1}^\top \times_{j+2} U_{j+1}^\top \cdots \times_{m+1} U_{m}^\top.
\]
Similar notation is defined for the partial contraction $\mY_{j}$ of $\cY_{1:n}$.
With this notation, we set the gradients
$\nabla_{U_j} \cL$ and $\nabla_\lambda \cL$ equal to zero,
which yields the following stationary conditions
\begin{align*}
\frac{1-2 \lambda}{n} \mX_{j}^\top \mX_{j} U_j = \frac{1}{n} \mX_{j}^\top \mY_{j} V_j, \ \ \
1-2 \lambda =  U_j^\top (\frac{1}{n} \mX_{j}^\top \mX_{j}) U_j.
\end{align*}
Combining the two equations with similar stationary conditions for
$V_j$, we obtain the following updating rule
\begin{equation}
  \label{eq:hopm:update}
    U_j = \frac{ \mX_{j}^\dagger \mY_{j} V_j }{ \sqrt{V_j^\top \mY_{j}^\top \mX_{j} \mX_{j}^\dagger \mY_{j} V_j} }, \quad
    V_j = \frac{ \mY_{j}^\dagger \mX_{j} U_j }{ \sqrt{U_j^\top \mX_{j}^\top \mY_{j} \mY_{j}^\dagger \mX_{j} U_j} },
\end{equation}
where $\mX_{j}^\dagger, \mY_{j}^\dagger$ are the pseudo inverses of
$\mX_{j}, \mY_{j}$, respectively. The update \eqref{eq:hopm:update} is in the form of the power method
on matrices $\mX_{j}^\dagger \mY_{j} \mY_{j}^\dagger \mX_{j}$ and
$\mY_{j}^\dagger \mX_{j} \mX_{j}^\dagger \mY_{j}$
\citep{Regalia2000higher, Kolda2009Tensor}. Cyclically updating each
component yields the higher-order power method, which is detailed in
Algorithm~\ref{alg:TCCA} and similar to the one in \citet{Wang2016Efficient}.

\begin{algorithm}[t]
	\SetAlgoLined
	\SetKwInOut{Input}{Input}
	\SetKwInOut{Output}{Output}
	\nl\Input{$\cX_{1:n}, \cY_{1:n}\in \R^{n \times d_1 \times \dots \times d_m }$, $r_x, r_y\geq0$, $\epsilon>0$}
	\nl \While{not converged}
	{
		\nl \For{$j=1,2,\dots,m$}
		{
		    \nl $\mX_{k j} = \cX_{1:n} \times_2 U_{k1}^\top \cdots \times_{j} U_{k,j-1}^\top \times_{j+2} U_{k-1,j+1}^\top \cdots \times_{m+1} U_{k-1,m}^\top$

		    \nl $\mY_{k j} = \cY_{1:n} \times_2 V_{k1}^\top \cdots \times_{j} V_{k,j-1}^\top \times_{j+2} V_{k-1,j+1}^\top \cdots \times_{m+1} V_{k-1,m}^\top$

			\nl {\bf option I (exact updating)}: $\tilde{U}_{k j} = (\mX_{kj}^\top \mX_{kj}+r_x \mI)^{-1} \mX_{kj}^\top \mY_{kj} V_{k-1,j}$

			\nl {\bf option II (inexact updating)}: $\tilde{U}_{kj}$ is an $\epsilon$-suboptimum of $\min_{\tilde{U}} \frac{1}{2n} \| \mX_{kj} \tilde{U} - \mY_{kj} V_{k-1,j} \|^2 + \frac{r_{x}}{2} \| \tilde{U} \|^2$

			\nl {\bf option I (HOPM)}: $U_{kj} = \tilde{U}_{kj}  (\tilde{U}_{kj}^\top (\frac{1}{n}\mX_{kj}^\top \mX_{kj}) \tilde{U}_{kj})^{-1/2} $

			\nl {\bf option II (sHOPM)}: $U_{kj} = \tilde{U}_{kj}  \|\tilde{U}_{kj} \|^{-1} $

			\nl  {\bf option I (exact updating)}: $\tilde{V}_{kj} = (\mY_{kj}^\top \mY_{kj}+r_y \mI)^{-1} \mY_{kj}^\top \mX_{kj} U_{kj}$

			\nl {\bf option II (inexact updating)}: $\tilde{V}_{k-1,j}$ is an $\epsilon$-suboptimum of $\min_{\tilde{V}} \frac{1}{2n} \| \mX_{kj} \tilde{U}_{kj} - \mY_{kj} \tilde{V} \|^2 + \frac{r_{y}}{2} \| \tilde{V} \|^2$

			\nl {\bf option I (HOPM)}: $V_{kj} = \tilde{V}_{kj} (\tilde{V}_{kj}^\top (\frac{1}{n} \mY_{kj}^\top \mY_{kj}) \tilde{V}_{kj})^{-1/2} $

			\nl {\bf option II (sHOPM)}:
			$V_{kj} = \tilde{V}_{kj} \| \tilde{V}_{kj} \|^{-1}$
		}
		\nl $k=k+1$
	}
	\nl \Output{$\cU_k$,  $\cV_k$}
	\caption{Higher-order Power Method}\label{alg:TCCA}
\end{algorithm}


In a CCA problem, only the projection subspace is identifiable, since the
correlation is scale invariant. Same is true in PCA and partial least squares
(PLS) problems in both matrix and tensor cases \citep{Arora2012Stochastic,
Uschmajew2015new, Gao2017Stochastic}. Therefore, normalization steps restricting
canonical variables are only important for numerical stability.  The major
difference between PCA and CCA is that the normalization steps are essential in
HOPM for preventing $\cU$ and $\cV$ iterates from converging to zero quickly.

The key insight is that different regularization schemes actually generate the
same sequence of correlation values for one component CCA, which is stated
formally in Proposition~\ref{prop:hopm-als}. This insight allows us to solve
TCCA via a variant, called Simple HOPM (sHOPM), which uses the Euclidean norm as
regularization, and prove existence and uniqueness of the limit of iterates of
HOPM. See Algorithm~\ref{alg:TCCA} for details. The Euclidean norm provides the
simplest form of regularization and does not depend on the data, which increases
the numerical stability as well as decreases computational costs.
One of the reasons for numerical instability of HOPM comes from rank deficiency
of the covariance matrices $\frac{1}{n}\mX_{kj}^\top \mX_{kj}$ and $\frac{1}{n}
\mY_{kj}^\top \mY_{kj}$, which commonly arises in CCA due to the low rank
structure of the data and undersampling. The numerical stability can also be
improved by adding a regularization term $r\mI$ to the empirical covariance
matrix.

Another major problem is related to identifiability.
Observe that if $\cU = U_1 \circ \cdots \circ U_{m}$ and
$\cV = V_1 \circ \cdots \circ V_{m} $ are a stationary point,
then $U_1/c \circ U_2 \circ \cdots \circ U_{m-1} \circ cU_{m}$ and
$cV_1 \circ V_2 \circ \cdots \circ U_{m-1} \circ V_{m}/c$ are also a stationary point
for any non-zero constant $c$. In particular, the optimum set is not isolated nor compact.
In this case, it is possible that the iteration sequence diverges,
even while approaching the stationary set. This is the main difficulty
in the analysis of convergence in tensor factorization.
Moreover, any component approaching zero or infinity causes
numerical instability. One way to overcome identifiability
is by adding a penalty function to balance magnitude of each component.
However, using this approach, one cannot find the exact solution
when updating each component and the whole optimization process
is slowed down. We show that HOPM for one component CCA is equivalent to projecting each
component to a compact sphere, so those remedies are unnecessary.

The projection to the unit sphere is not the projection
onto the constraint in \eqref{eq:tcca:sample2}.
Therefore, we need to address the question whether two different projections generate two
different sequences of iterates that have different behaviors.
In what follows, we explain how to answer the above question by introducing a new objective function.
Consider the modified loss (potential) function
\begin{equation}
\label{eq:mloss}
\tilde{\cL}(\alpha, \cU, \beta, \cV, \lambda, \mu) = \frac{1}{2 n} \sum_{t=1}^{n} (\ve{\alpha\cU}{\cX_t} - \ve{\beta\cV}{\cY_t})^2 + \lambda (1- \frac{1}{n} \sum_{t=1}^{n}\ve{\alpha\cU}{\cX_t}^2) +\mu (1 - \frac{1}{n} \sum_{t=1}^{n}\ve{\beta\cV}{\cY_t}^2),
\end{equation}
where we have added two extra normalization variables to the
components of $\cU$ and $\cV$. This type of potential function appears
in the literature on PCA and tensor decomposition. For example,
the following two optimization problems are equivalent for
the problem of finding the best rank-one tensor $\cU=U_1 \circ U_2 \cdots \circ U_m$ approximation of $\cX$
\[
	\min_{\cU, \alpha: \|U_1\| = \dots = \|U_m\| = 1} \| \cX - \alpha \cU \|^2
    \quad\Longleftrightarrow\quad
    \min_{\cU} \| \cX - \cU \|^2.
\]
HOPM represents the latter problem which is convex with respect to each
component $U_j$, and sHOPM represents the former optimization problem which is no longer
a convex problem with respect to $\alpha$ and $U_j$.
In the first formula, it is obvious that once $\cU$ is found,
then so is $\lambda$. Therefore, we may ignore $\alpha$ in
the optimization procedure, but considering this form is convenient
for the analysis.

The following proposition shows how the sHOPM relates to \eqref{eq:mloss}:
\begin{proposition} \label{prop:just-sHOPM}
    Suppose we have the dynamical iterates of sHOPM as follows
    \begin{equation}
    \label{ALSdynamic}
    \begin{split}
    U_{kj} & = \mX_{kj}^\dagger \mY_{kj} V_{k-1,j} / \| \mX_{kj}^\dagger \mY_{kj} V_{k-1,j} \|, \\
    \alpha_{kj} & = \left(  U_{k,j}^\top \mX_{kj}^\top \mX_{kj} U_{kj}\right)^{-1/2}, \\
    1-2\lambda_{kj} & = \alpha_{kj} \beta_{k-1,j} U_{kj}^\top \mX_{kj}^\top \mY_{kj} V_{k-1,j} = \rho_n(\cU_{kj}, \cV_{k-1,j}), \\
    V_{kj} & = \mY_{kj}^\dagger \mX_{kj} U_{kj} / \| \mY_{kj}^\dagger \mX_{kj} U_{kj} \|,\\
    \beta_{kj} & = \left(  V_{k,j}^\top \mY_{kj}^\top \mY_{kj} V_{k,j}\right)^{-1/2}, \\
    1-2\mu_{kj} & = \alpha_{kj} \beta_{kj} U_{kj}^\top \mX_{kj}^\top \mY_{kj} V_{k,j} = \rho_n(\cU_{kj}, \cV_{kj}).
    \end{split}
    \end{equation}
    Then \eqref{ALSdynamic} satisfies the stationary condition of \eqref{eq:mloss}.
\end{proposition}

The proof follows by simple calculation and is given in Appendix~\ref{sec:prof-prop-just-sHOPM}.
This gives a justification for sHOPM, which alternatively produces iterates
that satisfy the stationary condition of \eqref{eq:mloss}. Since we introduced
the normalization variables $\alpha, \beta$, we changed the
subproblem to a noncovex problem, we do not know
if sHOPM increases the correlation in each iteration or not. To answer
this, the following proposition shows that HOPM and sHOPM generate
iterates with the same correlation in each iteration based on the
fact that correlation is scale invariant.
In particular, this shows that HOPM increases the correlation
each time it solves the TCCA problem in \eqref{eq:tcca:sample2} regardless of regularization.

\begin{proposition} \label{prop:hopm-als}
	Let $(U_{k j}, V_{k j})$, $(A_{kj},B_{kj})$ be the iterates generated by HOPM option I and II with the same starting values, respectively. Then, it holds that
	\[
	U_{kj} = \frac{A_{kj}}{ \sqrt{A_{kj}^\top (\frac{1}{n}\mX_{kj}^\top \mX_{kj}) A_{kj}} } ,
\quad  V_{kj} = \frac{B_{kj}}{ \sqrt{B_{kj}^\top (\frac{1}{n}\mY_{kj}^\top \mY_{kj}) B_{kj}} },
	\]
	and $\rho_n(\cU_{k j}, \cV_{k j}) = \rho_n (\cA_{k j}, \cB_{k j})$.
	Moreover, if $(\alpha, A_1, \dots, A_m, \lambda, \beta, B_1, \dots, B_m, \mu)$
	is a critical point of the modified loss $\tilde{\cL}$ and
	$\alpha,\beta>0$, then $(\alpha^{(1/m)}A_1, \dots, \alpha^{(1/m)}A_m, \beta^{(1/m)}, \lambda, B_1, \dots, \beta^{(1/m)}B_m, \mu)$
	is a critical point of the original loss $\cL$.
\end{proposition}

Note that it is easier to understand the power method from a
linear algebra perspective, as it amplifies the largest eigenvalue at each iteration,
instead of an optimization perspective. However, we find the modified
Lagrange function useful for using the Lojasiewicz gradient property.

\citet{Ma2015Finding} used a similar idea to develop a faster algorithm
for CCA, called AppGrad. The difference is that we establish the
relationship between two alternating minimization schemes, while
\citet{Ma2015Finding} established the relationship between
gradient descent schemes. Notably, they only show that CCA is a fixed
point of AppGrad, while we further illustrate that this type of
scheme actually finds a stationary point of a modified non-convex loss \eqref{eq:mloss}.
Thus HOPM and AppGrad are nonconvex methods for
the CCA problem. HOPM is also related to the optimization
on matrix manifold \citep{Absil2008Optimization}, but the
discussion is beyond the scope of this paper.

\subsection{Convergence Analysis}
\label{sec:convergence-analysis}

To present our main convergence result for sHOPM, 
we start by introducing several assumptions under which
we prove the convergence result.
\begin{assumption}
\label{assumption1}
	Assume the following conditions hold:
	\begin{align*}
	    & 0< \sigma_{l,x} =: \sigma_{\min}\left(\frac{1}{n} \sum_{t=1}^{n} \mbox{vec}(\cX_t)\mbox{vec}(\cX_t)^\top\right)<\sigma_{\max}\left(\frac{1}{n} \sum_{t=1}^{n} \mbox{vec}(\cX_t)\mbox{vec}(\cX_t)^\top\right) := \sigma_{u,x}<\infty, \\
	    & 0<\sigma_{l,y} =: \sigma_{\min}\left(\frac{1}{n} \sum_{t=1}^{n} \mbox{vec}(\cY_t)\mbox{vec}(\cY_t)^\top\right)<\sigma_{\max}\left(\frac{1}{n} \sum_{t=1}^{n} \mbox{vec}(\cY_t)\mbox{vec}(\cY_t)^\top\right) := \sigma_{u,y}<\infty, \\
	    &\rho_n(\cU_0, \cV_0)>0.
	\end{align*}
\end{assumption}
The same conditions appeared in \cite{Ma2015Finding},
who studied CCA for vector-valued data. \cite{Wang2016Efficient} require the smallest
eigenvalue to be bounded away from zero, which can always be achieved
by adding the regularization term to the covariance matrix, as
discussed in the previous section.
However, instead of assuming a bound on the largest eigenvalue,
they assume that $\max_i \|x_i\|$ and $\max_i \|y_i\|$ are bounded.
The third condition can easily be satisfied by noting that if
$\rho_n(\cU_0, \cV_0)<0$, we can flip the signs of the components
of $\cU_0$ or $\cV_0$ to obtain $\rho_n(\cU_0, \cV_0)>0$.
Finally, we remark that the first two conditions are sufficient
for preventing the iterates from converging to zero.
See Lemma~\ref{lemma:boundedsolution} for details.

\begin{theorem} \label{localConv}
	If Assumption \ref{assumption1} holds, then the dynamic \eqref{ALSdynamic}
	satisfies Conditions \eqref{A1}, \eqref{A2}, and \eqref{A3}.
	Furthermore, the iterates $U_{kj}, V_{kj}$ generated by sHOPM converge
	to a stationary point at the rate that depends on the exponent in
	the Lojasiewicz gradient inequality in Lemma~\ref{lemma:Lojineq}.
\end{theorem}

Without arbitrary initial points and any explicit model assumptions,
Theorem~\ref{localConv} establishes a convergence rate that depends on the
exponent in the Lojasiewicz gradient inequality determined by the data. We only
require the data to be well conditioned. Although this analysis does not give us
an exact convergence rate, \citet{Espig2015}, \citet{Espig2015a}, and
\citet{Hu2018Convergence} indicated that sublinear, linear and superlinear
convergence can happen in the problem of rank-one approximation of a tensor. For
the tensor canonical correlation analysis, we show that with a stronger
assumption on a data generating process, it is possible to get linear
convergence. See Section~\ref{sec:p2DCCA}. Note that \cite{liue16Quadratic}
showed the exponent in a  Lojasiewicz inequality is $1/2$ in quadratic
optimization with orthogonality constraints which is the case of the matrix
decomposition such as PCA. From the previous discussion, the exponent can be any
number between $0$ and $1$ in the problem of tensor decomposition, which
illustrates convergence rates in the matrix and tensor decompositions are
extremely different from each other. This fact also points out that
Theorem~\ref{localConv} is the optimal in the sense that we cannot determine the
convergence rate in theory without extra assumption.

\subsection{{$(k_1,k_2,\dots,k_m)$}-TCCA and Deflation}
\label{sec:gTCCA}

In this section, we develop a general TCCA procedure for extracting
more than one canonical component. We can interpret general TCCA as a
higher rank approximation of general CCA. That is, we seek to solve the following CCA problem
\begin{equation}
\label{eq:gtcca}
\min_{\cU, \cV \text{ in a low-rank space}} \frac{1}{n} \sum_{t=1}^{n} (\ve{\cU}{\cX_t} - \ve{\cV}{\cY_t})^2
\qquad\text{ s.t. }\qquad
\frac{1}{n} \sum_{t=1}^{n}\ve{\cU}{\cX_t}^2 = 1 = \frac{1}{n} \sum_{t=1}^{n}\ve{\cV}{\cY_t}^2,
\end{equation}
where $\cU, \cV$ lie in a ``low-rank" space. For example, TCCA restricts solutions $\cU, \cV$ in
the space of rank-one tensors. There are many ways to obtain a higher rank tensor factorization,
but here we focus on rank-$(r_1, r_2, \dots, r_m)$ approximation \citep{DeLathauwer2000best},
which is particularly related to 2DCCA and TCCA.
To present the corresponding extension of HOPM for {$(k_1,k_2,\dots,k_m)$}-TCCA,
define
\begin{align*}
	\mX_{kj} =& \sum_{t=1}^n (\cX_{t})_{(j+1)} (\mU_{k-1,m} \dots \otimes \mU_{k-1,j+1} \otimes \mU_{k,j-1} \otimes \dots \mU_{k,1}), \\
	\mY_{kj} =& \sum_{t=1}^n (\cY_t)_{(j+1)} (\mV_{k-1,m} \dots \otimes \mV_{k-1,j+1} \otimes \mV_{k,j-1} \otimes \dots \mV_{k,1}).
\end{align*}
Then we could use following updating
\begin{equation} \label{eq:gHOPM}
\begin{split}
	\tilde{\mU}_{k j} =& (\mX_{kj}^\top \mX_{kj})^{-1} \mX_{kj}^\top \mY_{kj} \mV_{k-1,j}, \\
	\mU_{kj} =& \tilde{\mU}_j  \left(\tilde{\mU}_{kj}^\top \left(\frac{1}{n}\mX_{kj}^\top \mX_{kj} \right) \tilde{\mU}_{kj} \right)^{-1/2}, \\
	\tilde{\mV}_{k j} =& (\mY_{kj}^\top \mY_{kj})^{-1} \mY_{kj}^\top \mX_{kj} \mU_{kj}, \\
	\mV_{kj} =& \tilde{\mV}_j \left(\tilde{\mV}_{kj}^\top \left(\frac{1}{n} \mY_{kj}^\top \mY_{kj}\right) \tilde{\mV}_{kj} \right)^{-1/2}.
\end{split}
\end{equation}
Here we replace vectors by matrices and only need to compute the SVD for small matrices
$\tilde{\mU}_{kj}^\top (\frac{1}{n}\mX_{kj}^\top \mX_{kj}) \tilde{\mU}_{kj}$
and
$\tilde{\mV}_{kj}^\top (\frac{1}{n} \mY_{kj}^\top \mY_{kj}) \tilde{\mV}_{kj}$.

The main concern with HOPM is that it may not have a feasible solution due to the fact that
the orthogonal relationship in high rank tensor space may not be
well-defined. For example, it may not possible to find $\mR_x, \mL_x, \mR_y, \mL_y$ that satisfy
the $(k_1,k_2)$-2DCCA constraints
in general:
\begin{equation}
\label{eq:k1k22DCCAobj}
\begin{split}
& \E{(\mR_x \otimes \mL_x)^\top \text{vec}(\mX) \text{vec}(\mX)^\top (\mR_x \otimes \mL_x)} = \mI, \\
& \E{(\mR_y \otimes \mL_y)^\top \text{vec}(\mY) \text{vec}(\mY)^\top (\mR_y \otimes \mL_y)} =\mI.
\end{split}
\end{equation}


In section~\ref{sec:p2DCCA}, we consider a probabilistic data generating process
with a low rank structure. In this setting, we show that  it is possible to find
a solution using \eqref{eq:gHOPM}. In practice, when the data generating process
is unknown, we can still use HOPM as a non-convex method for low-rank
approximation with a relaxation of CCA constraints up to some noise or error.

Without assuming a probabilistic low-rank data generating model, it is not clear
whether there is a feasible solution to \eqref{eq:k1k22DCCAobj} as discussed
above. Hence, we discuss a {\bf deflation} procedure here that can be used for
extracting more than one canonical component \citep{Kruger2003Canonical,
Sharma2006Deflation}. Deflation procedure is summarized in
Algorithm~\ref{deflationTCCA} and closely related to CP decomposition. Moreover,
unlike \eqref{eq:gHOPM}, there is no need for any computation for SVD and by
Proposition~\ref{prop:hopm-als} we could use the simple projection.

\begin{algorithm}[t]
	\SetAlgoLined
	\SetKwInOut{Input}{Input}
	\SetKwInOut{Output}{Output}
	\nl\Input{$\cX_{1:n}, \cY_{1:n} \in \R^{n\times d_1 \dots \times d_m}$, $r$}


	\nl \While{not converged}
	{
		\nl \For{$k=1,2,\dots,r$}
		{
			\nl use $\hat{\cX}_k, \hat{\cY}_k$ update $\cU_k, \cV_k$ by Algorithm \ref{alg:TCCA}

			\nl Compute the residual $\hat{\cX}_k= \cX_{1:n} \times_1 (\mI-T_{k} T_{k}^\top)$, $\hat{\cY}_k= \cY_{1:n}, \times_1 (\mI-S_{k} S_{k}^\top) $ where
			$T_k = \begin{bmatrix}
			\ve{\cU_1}{\cX_1}& \dots& \ve{\cU_1}{\cX_n}& \\
			\vdots && \vdots \\
			\ve{\cU_{k-1}}{\cX_1}& \dots& \ve{\cU_{k-1}}{\cX_n}& \\
			\ve{\cU_{k+1}}{\cX_1}& \dots& \ve{\cU_{k+1}}{\cX_n}& \\
			\vdots && \vdots \\
			\ve{\cU_r}{\cX_1}& \dots& \ve{\cU_r}{\cX_n}& \\
			\end{bmatrix}^\top,
			S_k = \begin{bmatrix}
			\ve{\cV_1}{\cY_1}& \dots& \ve{\cV_1}{\cY_n}& \\
			\vdots && \vdots \\
			\ve{\cV_{k-1}}{\cY_1}& \dots& \ve{\cV_{k-1}}{\cY_n}& \\
			\ve{\cV_{k+1}}{\cY_1}& \dots& \ve{\cV_{k+1}}{\cY_n}& \\
			\vdots && \vdots \\
			\ve{\cV_r}{\cY_1}& \dots& \ve{\cV_r}{\cY_n}& \\
			\end{bmatrix}^\top$

		}
	}
	\nl \Output{$\cU_1, \dots, \cU_r$,  $\cV_1, \dots, \cV_r$}
	\caption{Deflation for TCCA}\label{deflationTCCA}
\end{algorithm}

In order to see how deflation works, assume the simplest case with $k_1=2=k_2$.
Let $\mU=(U_1, U_2)$ and $U_1=U_{11} \otimes U_{12}, U_2=U_{21} \otimes U_{22}$. Then $U_2$ satisfy the following equation
\[
\frac{1}{n} \sum_t U_1^\top\Vector{(\hat{\cX}_t)} \Vector{(\hat{\cX}_t)}^\top U_2=0,
\]
where $\hat{\cX}_t = (\mI - T_1T_1^\top) \Vector{(\cX_t)}$
is the projected data with
$T_1= (U_1^\top\Vector{(\cX_1)}, \dots, U_1^\top\Vector{(\cX_n)})^\top$.
Thus, using the projected idea is similar to
uncorrelated constraints in the CCA problem. We can optimize $U_1, U_2$
in an alternating fashion and similarly for $\mV$.
Indeed, this is not exactly solving the CCA problem,
but it is a relaxed version of CCA like HOPM for ($k_1,\dots,k_m$)-TCCA.
In practice, even relaxed constraint can improve performance
compared to sample constraint methods such as Partial Least Squares (PLS).
See Section~\ref{genodata} for an application to a genotype data.

\section{Practical considerations}
\label{sec:practical-considerations}

In this section, we discuss computational issues associated with
TCCA. Inexact updating rule and several useful schemes are
included.

\subsection{Efficient Algorithms for Large-scale Data}
\label{sec:efficient-algo}

The major obstacle in applying CCA to large scale data is
that many algorithms involve inversion of large matrices,
which is computationally and memory intensive. This problem
also appears in Algorithm~\ref{alg:TCCA}. Inspired by the
inexact updating of 1DCCA, we first note that that
$\tilde{U}_{k j} = (\mX_{kj}^\top \mX_{kj})^{-1} \mX_{kj}^\top \mY_{kj} V_{k-1,j}$ is
the solution to the following least squares problem
\begin{equation}
\label{eq:LSPTCCA}
   \min_{\tilde{U}} \frac{1}{2n} \| \mX_{kj} \tilde{U} - \mY_{kj} V_{k-1,j} \|^2,
\end{equation}
and similarly for $\tilde{V}_{k j}$.
In the following theorem, we show that it suffices to solve the
least squares problems inexactly. As long as we can bound the error
of this inexact update, we will obtain sufficiently accurate
estimate of canonical variables. More specifically, we show the
error accumulates exponentially. See Algorithm~\ref{alg:TCCA} for
the complete procedure with inexact updating for HOPM.
Theorem~\ref{thm:inexactupdating} allows us to use advanced
stochastic optimization methods for least squares problem, e.g.,
stochastic variance reduced gradient \citep{Johnson2013Accelerating,
Shalev-Shwartz2013Stochastic}.

\begin{theorem}
    \label{thm:inexactupdating}
	Denote $\{U^\ast_{kj},V^\ast_{kj}\}$ generated by option I in Algorithm \ref{alg:TCCA} and $U_{kj},V_{kj}$ generated by sHOPM in Algorithm \ref{alg:TCCA}. Under Assumption~\ref{assumption1}, we have
	\[
	\max \{\| U_{kj} - U_{kj}^{\ast} \|, \| V_{kj} - V_{kj}^{\ast} \|\} = O(r^{2mk+j}\sqrt{\epsilon}),
	\]
	for some $r$ that depends on $m, \sigma_{u,x}, \sigma_{l,x}, \sigma_{u,y}, \sigma_{l,y}$.
\end{theorem}
Note that we obtain the same order for the error bound for the inexact
updating bounds as in \cite{Wang2016Efficienta}, who studied the case with
$m=1$.






Several techniques can be used to speed the convergence of inexact updating of HOPM.
First, we can use warm-start to initialize the least squares solvers by setting
initial points as $\tilde{U}_{k-1,j}, \tilde{V}_{k-1,j}$. Due to Theorem~\ref{localConv}, after several iterations, the subproblem \eqref{eq:LSPTCCA} of TCCA varies
slightly with $U_{kj}$ and $V_{kj}$, i.e., for large enough $k$
\[
    \|U_{kj} - U_{k-1,j}\| \approx 0 \approx \|V_{kj} - V_{k-1,j}\|.
\]
Therefore we may use
$\tilde{U}_{k,j-1}$ as an initialization point when minimizing \eqref{eq:LSPTCCA}.
Second, we can regularize the problem by adding the ridge penalty, or equivalently
by setting $r_x,r_y>0$ in Algorithm~\ref{alg:TCCA}. The $\ell_2$ regularization makes
the least squares problem guaranteed to be strongly convex and speeds up the
optimization process. This type of regularization is necessary when the size
of data is smaller than the dimension of parameters for the condition about
smallest eigenvalue in Assumption~\ref{assumption1} to be satisfied \citep{Ma2015Finding,Wang2018Efficient}.
Finally, the shift-and-invert preconditioning method can also be  considered, but we leave this
for future work.

\subsection{Effective Initialization}
\label{sec:effetive-initial}

In this section, we propose an effective initialization procedure for the
 $m=2$ case, focusing on the 2DCCA problem.
Since the 2DCCA problem is non-convex, there are no guarantees that HOPM converges to
a global maximum. Therefore, choosing an initial point is important.
We propose to initialize the procedure via CCA,
\[
(C_x, C_y) = \arg \max_{C_x, C_y} \mbox{corr}(C_x^\top \mbox{vec}(\cX), C_y^\top \mbox{vec}(\cY)),
\]
and use the best rank-1 approximation as the initialization point.
More specifically, we find $U_1, U_2, V_1, V_2$ such that
$U_2 \otimes U_1 $ and $V_2 \otimes V_1$ are the best approximations of $C_x$ and $C_y$,
which can be obtained by SVD of $\mbox{unvec}(C_x)$ and $\mbox{unvec}(C_y)$.
Heuristically, an initial point using the best rank-1 approximation may have higher
correlation than that of a random guess and, therefore, it is more likely to be close
to a global maximum.

Under the p2DCCA model in \eqref{eq:p2DCCA}, we showed in the last section
that the 2DCCA can find the optimum of \eqref{eq:k1k22DCCAobj}. Under this model,
CCA and 2DCCA coincide at the population level. Therefore, as $n$ increases,
the CCA solution approaches the global optimum of 2DCCA and it is reasonable to
use this as an initialization.

\section{Numerical Studies}
\label{sec:num-studies}

In this section, we carefully examine convergence properties of TCCA and our theorems
via simulation studies and empirical data analysis. A comparison to other methods is included. We also study 
the effect of the initialization scheme discussed in Section~\ref{sec:effetive-initial}. 
Code and data can be found in github: \url{https://github.com/youlinchen/TCCA}.

We first examine TCCA on synthetic data. Consider the p2DCCA model of \eqref{eq:p2DCCA} for $t=1,\dots,n,$
\begin{equation}
\label{eq:p2DCCA:simulation}
\begin{split}
X_t & = \mPhi_1 ({\bf \Lambda_1} \odot \mC_t + {\bf \Lambda}_2 \odot \mE_{xt}) \mPhi_2^\top, \\
Y_t & = \mOmega_1 ({\bf \Lambda_1} \odot \mC_t + {\bf \Lambda}_2 \odot \mE_{yt}) \mOmega_2^\top,
\end{split}
\end{equation}
where $\odot$ denotes the entry-wise (Hadamard) product,
$\mPhi_1 \in \R^{m_x\times k}, \mPhi_2 \in \R^{n_x \times k}, \mOmega_1 \in \R^{m_y  \times k}, \mOmega_2 \in \R^{n_y \times k}$, and
$\mPhi_1, \mPhi_2, \mOmega_1, \mOmega_2$ are generated randomly to satisfy
$\mPhi_1^\top\mPhi_1 = \mI = \mPhi_2^\top\mPhi_2 = \mOmega_1^\top \mOmega_1 =\mOmega_2^\top \mOmega_2$. We achieve this by first generating matrices with elements being random
draws from $N(0,1)$ and then performing the QR decomposition.
${\bf \Lambda}_i$ ($i=1,2$) are fixed matrices whose elements
are between 0 and 1. In the following simulations,
we assume the simple case that $k=2, m_x=20, n_x=15, m_y=15, n_y=20$
and the elements of $\mC_t, \mE_{xt}, \mE_{yt}$ are random draws from $N(0,1)$, and
\[
{\bf \Lambda_1} =
\begin{bmatrix}
\sqrt{\lambda} & 0 \\
0 & 0
\end{bmatrix},\quad
{\bf \Lambda_2} =
\begin{bmatrix}
\sqrt{1-\lambda} & 1 \\
1 & 1
\end{bmatrix}.
\]
The population optimum of CCA and
2DCCA coincide by Proposition~\ref{prop:optimum-1DCCA-2DCCA}, and the population optimal correlation is $\lambda$.

For the first experiment, we generate $n=100$ samples from
\eqref{eq:p2DCCA:simulation} with $\lambda=0.9$ and apply the
HOPM, sHOPM 100 times to this data set
with 100 different random initializations.
To test Theorem~\ref{localConv}, we check the norm of the difference between consecutive
loadings for each iteration $k$:
\begin{equation}
    \label{eq:def:diff}
    \operatorname{diff}(k) = \| U_{1,k} - U_{1,k-1} \|_2 + \| U_{2,k} - U_{2,k-1} \|_2 + \| V_{1,k} - V_{1,k-1} \|_2 + \| V_{2,k} - V_{2,k-1} \|_2.
\end{equation}

\begin{figure}[h]
	\centering
	\begin{subfigure}{0.32\textwidth}
		\includegraphics[width=\textwidth]{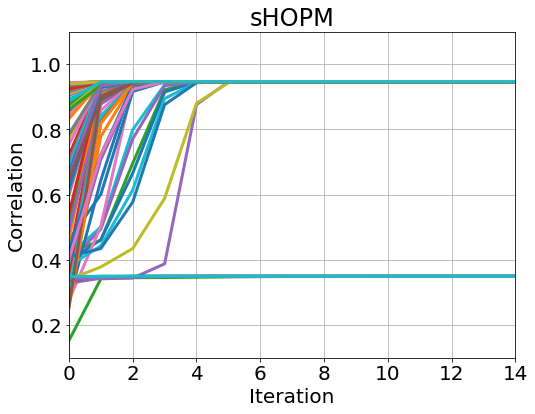}
	\end{subfigure}
	\begin{subfigure}{0.32\textwidth}
		\includegraphics[width = \textwidth]{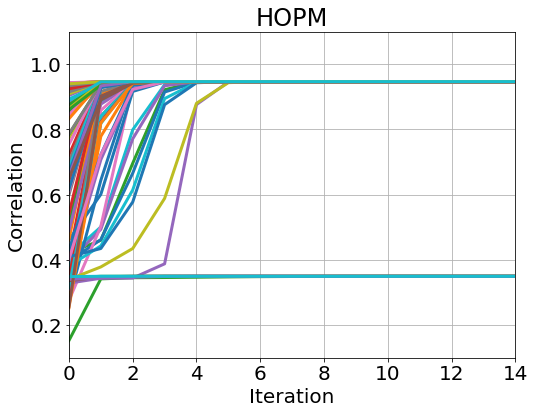}
	\end{subfigure}  \\
	\begin{subfigure}{0.32\textwidth}
		\includegraphics[width=\textwidth]{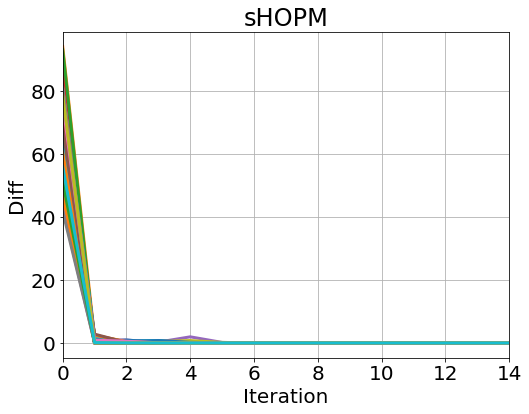}
	\end{subfigure}
	\begin{subfigure}{0.32\textwidth}
		\includegraphics[width = \textwidth]{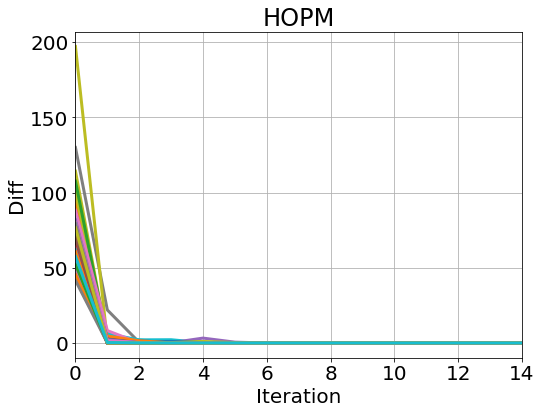}
	\end{subfigure}
	\caption{The first row shows the correlation \eqref{eq:def:corr} in each iteration, while
	the second row shows the difference \eqref{eq:def:diff} in each iteration. The plots illustrate
	various methods for 100 random initializations on the same data set. This figure
	reveals that sHOPM and HOPM have identical paths of correlation and the same convergence
	property, as explained by our theory. All methods suffer local optimums.}\label{fig:p2DCCA:lambda1}
\end{figure}

The results are shown in Figure~\ref{fig:p2DCCA:lambda1}. From the
plots, we see that sHOPM and HOPM exhibit the same convergence
property and have the identical path of correlations.

Next, we study whether iterates find the global optimum
using  different sample sizes $n=50, 100, 300, 700, 1000, 1500$
and different signal-to-noise ratios $\lambda=0.8, 0.5, 0.2$.
Since we do not know the finite-sample closed-form solution of TCCA, we
do the following approach to estimate the empirical
probabilities of attaining the global maximum, inspired
by observing that the global maximum is attained by most initializations
in Figure~\ref{fig:p2DCCA:lambda1}. We first generate a new dataset each
time and run, for a given data set, sHOPM 15 times with random initialization and treat the resulting
maximum correlation as the global maximum. Then, for a given estimation algorithm and a given
initialization,
we treat the algorithm as a success if the estimated correlation
is close to the global maximum of that dataset, say, with error $\epsilon = 0.01$.
We also compute the distance between the estimated components
$\hat{U_1}, \hat{U_2}, \hat{V_1}, \hat{V_2}$ generated by
the HOPM, sHOPM and the true population components $U_1,U_2,V_1,V_2$, where the error is defined by
\[
\operatorname{error}=\operatorname{error}(U_1, \hat{U_1})+\operatorname{error}(U_2, \hat{U_2})+\operatorname{error}(V_1, \hat{V_1})+\operatorname{error}(V_2, \hat{V_2}),
\]
with $\operatorname{error}(U_1, \hat{U_1}) = 1 - (U_1^\top \hat{U_1})^2$.
The results are summarized in Figure~\ref{fig:p2DCCA}.
From the plots, the distance between the estimated and the true population
loading goes to $0$ when the sample size increases. The plots also show
that an effective initialization
not only improves the probability of achieving the global optimum
but also reduces the average distance between the true and sample loadings.
Moreover, as the signal-to-noise increases the probability of attaining
the global optimum increases and the optimal correlation calculated by
running TCCA with 15 different random initializations approaches the population
correlation. Finally, our simulation results show that the probability of
achieving the global maximum increases as the sample size increases.

\begin{figure}[t]
	\centering
	\begin{subfigure}{0.32\textwidth}
		\includegraphics[width=\textwidth]{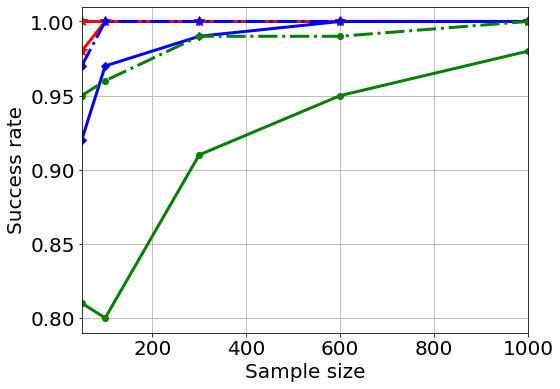}
	\end{subfigure}
	\begin{subfigure}{0.32\textwidth}
		\includegraphics[width = \textwidth]{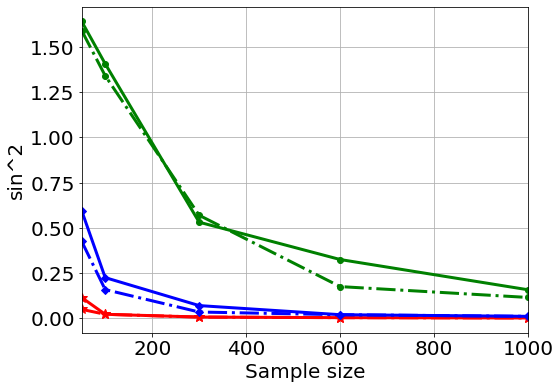}
	\end{subfigure}
	\begin{subfigure}{0.2\textwidth}
		\includegraphics[width = \textwidth]{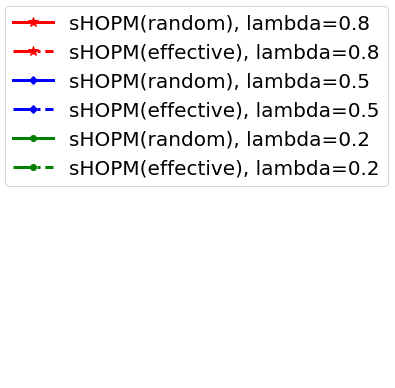}
	\end{subfigure}
	\caption{The left figure presents the success rate of achieving global
	optimum under various conditions obtained by running sHOPM 15 times in
	advance with random initialization. The right figure presents the
	distance between estimated components and the population solution.
	Each point is averaged over 1000 simulation iterations.
	These figures show that effective initialization
	improves convergence significantly.}\label{fig:p2DCCA}
\end{figure}

We compare three different methods to illustrate the superior performance of
TCCA, including AppGrad of \cite{Ma2015Finding} and the
truncated Rayleigh flow method (Rifle) of \cite{Tan2018Sparse}. AppGrad solves
1DCCA, while Rifle aims to solve sparse generalized eigenvalue problem.
Three datasets from different applications are included.
The detail of data can be found in the following sections.
See Section~\ref{genodata} for Gene Expression and Section~\ref{Appendix:experiments} for Adelaide.
MNIST is a database of handwritten digits. The goal is to learn correlated representations
between the up and low halves of the images. We randomly select 5000 features
of gene expression and genomics for reducing dimension. The datasets are separated
by a training set and a testing set to test generalization. The result is
presented in Table~\ref{table:comparsion}. We can see that TCCA outperforms both 
AppGrad and Rifle in the Adelaide dataset which has 355 sample and 336 features,
and especially in gene expression dataset for which the number
of features (p=5000) is much larger than the sample size (n=286). The 
three methods have comparable performance in MINST, whose sample size (n=60000)
is much larger than features (p=392).

\begin{table}[t]
	\centering
\begin{tabularx}{\textwidth}{YYYYYY}
	\hline
	 & & 1DCCA & TCCA & AppGrad & Rifle \\ \hline
	Adelaide             & corr (train)         & 0.994                     & 0.972                     & 0.952                       & 0.793                     \\
	& corr (test)          & 0.904                     & 0.969                     & 0.844                       & 0.836                     \\
	& time(s)              & 0.037                     & 0.097                     & 3.402                       & 131.1                     \\ \hline
	MINST                & corr(train)          & 0.962                     & 0.939                     & 0.953                       & 0.953                     \\
	& corr (test)          & 0.962                     & 0.943                     & 0.952                       & 0.952                     \\
	& time(s)              & 0.379                     & 1.181                     & 17.50                       & 279.5                     \\ \hline
	Gene   & corr(train)          & 1.000                     & 0.979                     & 0.949                       & 0.019                     \\
	Expression & corr (test)          & 0.393                     & 0.833                     & 0.233                       & 0.180                     \\
	& time(s)              & 40.32                     & 0.089                     & 39.94                       & 34657                     \\ \hline
\end{tabularx}
\caption{The summary of three methods on three datasets. 1DCCA is the baseline denoting solving \eqref{gepCCA} directly by SVD and time denotes the computing time in seconds.} \label{table:comparsion}
\end{table}

\subsection{Applications}

In this section, we consider three applications of TCCA to demonstrate
the power of the proposed analysis in reducing the  computational cost and in
revealing the data structure.

\subsubsection{Gene Expression Data} \label{genodata}

It is well known that genotype data mirror the population structure
\citep{Novembre2008Genes}. Using principal component analysis (PCA), single
nucleotide polymorphism (SNP) data of individuals are projected into
the first two principal components of the population-genotype matrix
and the location information can be recovered. Figure~\ref{expdata}(a) shows
populations are well-separated by the first two PCA components of genotype data. However, this technique
cannot be applied to other types of genomics data. From Figure~\ref{expdata}(b),
we can see there is no clear cluster using the first two PCA components of gene expression data.
In a recent paper,
\citet{Brown2018Expression} combined PCA and CCA to overcome this difficulty and
reveal certain population structure in gene expression data. The
authors notice that the failure of PCA in reconstructing geographical
information of the data is caused by the fact that data collected from
different laboratories are correlated. They regress the gene
expression matrix on data from different laboratories to correct the
confounding effect and then perform PCA analysis to extract the first few
principal components. They then apply CCA on the
batch-corrected expression data and principal components of genotype
data to achieve separation of the population in expression data. See
Figure~\ref{expdata}(c).

\begin{figure}[t]
	\centering
	\begin{subfigure}[b]{0.32\textwidth}
		\includegraphics[width=\textwidth]{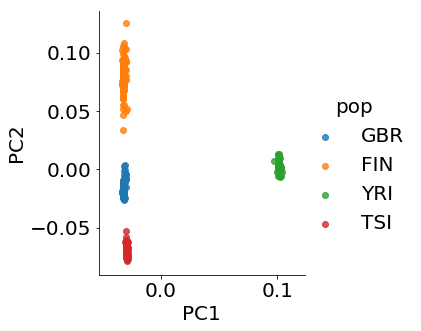}
		\caption{PCA for genotype data}
	\end{subfigure}
	\begin{subfigure}[b]{0.32\textwidth}
		\includegraphics[width=\textwidth]{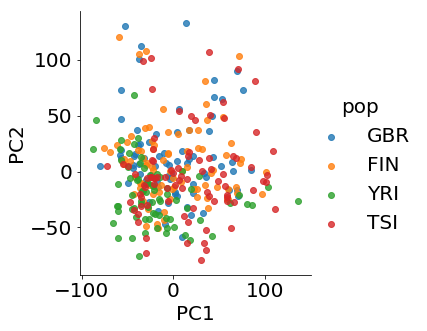}
		\caption{PCA for gene expression data}
	\end{subfigure} \\
	\begin{subfigure}[b]{0.32\textwidth}
		\includegraphics[width=\textwidth]{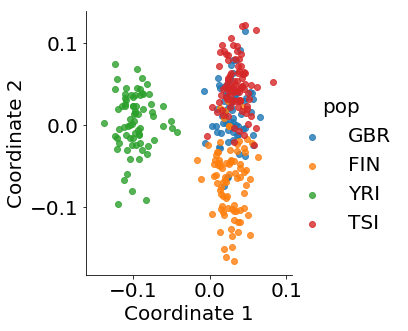}
		\caption{PCA+CCA}
	\end{subfigure}
	\begin{subfigure}[b]{0.32\textwidth}
		\includegraphics[width=\textwidth]{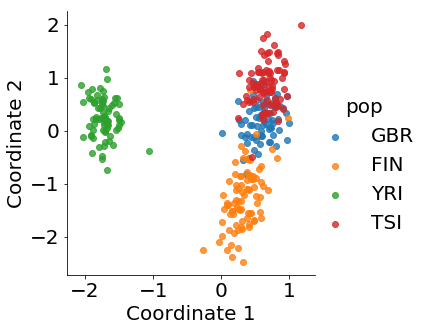}
		\caption{TCCA}
	\end{subfigure}
	\begin{subfigure}[b]{0.32\textwidth}
		\includegraphics[width=\textwidth]{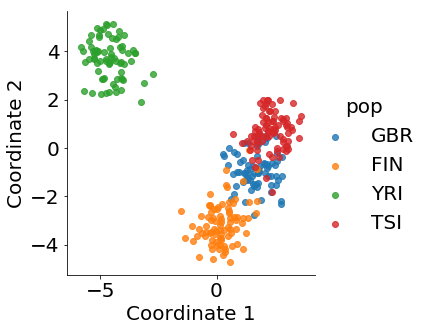}
		\caption{TCCA+deflation}
	\end{subfigure}
	\caption{The population structure of genotype and gene expression data}\label{expdata}\
\end{figure}

Intuitively, principal components of the genotype data are informative
in population cluster and can guide the expression data to
split the population via CCA.  Thus, it is not surprising to see the
distinct population patterns in the CCA projection of the expression
data.  In addition, PCA also serves as a dimension reduction tool to
reduce the computational cost.  This is essential because the original
genotype data contain around 7 million SNPs.

To exploit the information as much as possible, our goal in this
example is to use CCA to achieve
similar separation
without the preprocessing data with PCA.  
To this end, we reformulate the expression data
and genotype data as matrices and perform TCCA directly. For
illustration, we use 318 individuals with genotype data from 1000
genomes phase 1 and corresponding RNA-seq data from GEUVADIS in four
populations, GBR, FIN, YRI and TSI for comparison, 
which are the same as those in
\cite{Brown2018Expression}. We follow the same procedure given in
\cite{Brown2018Expression} to extract expression data and remove the
confounding. This left 14,079 genes in expression data, which we represent
as a 361$\times$39 matrix. The Phase-1 1000 genomes genotypes
contain 39,728,178 variants. We use LD pruning, which uses a moving
window to compute pairwise correlation and removes highly correlated
SNPs. This results in 738,192 SNPs that are formed into a
1014$\times$728 matrix. Finally, we perform TCCA and the results are
shown in Figure \ref{expdata}(d). The plot clearly shows
that TCCA improves the separation.
This is encouraging, because our method utilizes all information and
only takes less than half minute to run. Moreover, Figure~\ref{expdata}(e)
shows that TCCA+deflation provides some further improvement.

\subsubsection{Air Pollution Data in Taiwan}

In this example we use TCCA to analyze air pollution data of Taiwan. The
question of interest is whether and how the geographical and
meteorological factors affect air pollution.  The monthly average
data of various air pollution measurements
are downloaded from the website of the Environmental Protection
Administration, Executive Yuan, Taiwan. We use the data from 2005 to
2017 for a total of 156 months, 12 monitoring stations, and 7
pollutants.  The pollutants are sulfur dioxide (SO2), carbon monoxide
(CO), ozone (O3), particulate matter PM10, oxides of nitrogen (NOx),
nitric oxide (NO), and nitrogen dioxide (NO2).  The measurements of
each pollutant in a station are treated as a univariate time series,
and we employ a fitted seasonal autoregressive integrated moving-average
(ARIMA) model to remove the seasonality. This results in 144 months of
seasonally adjusted data for our analysis. Use of seasonally
adjusted data is common in economic and environmental studies. The
12 monitoring stations are Guting, Tucheng, Taoyua, Hsinchu, Erlin,
Xinying, Xiaogang, Meinong, Yilan, Dongshan, Hualien, and Taitung.
See the map in Section~\ref{TWmap} of the Appendix.

To examine the impact of geographical factors, we divide Taiwan into
north (Guting, Tucheng, Taoyuan and Hsinchu), south (Erlin, Xinying,
Xiaogang and Meinong), and east (Yilan, Dongshan, Hualien and Taitung)
regions. There are 4 stations in each region. Again, see Appendix~\ref{TWmap}.
Consequently, for this application, we have 144 months
by 7 pollutants by 4 stations in each region.  We then perform TCCA
between regions. To avoid getting trapped in a local maximum, we repeat TCCA
20 times and select the result with the highest correlation as the
solution. Tables~\ref{corrG}, \ref{loadingStationsG}, and
\ref{loadingPollutantG} summarize the results.

\begin{table}[p]
	\centering
	\begin{tabularx}{\textwidth}{YYY}
		\hline
		North vs South &    South vs East &   North vs East \\
		0.888 &  0.817 &  0.904 \\
		\hline
	\end{tabularx}
	\caption{The maximum correlations of the data between pairs of regions for Taiwan air pollutants.} \label{corrG}
\end{table}

\begin{table}[p]
	\centering
	\begin{tabularx}{\textwidth}{Y|YYYY|YYYY}
		\hline
		&&North&&&&South&&\\
		\hline
		N vs S &Guting &   Tucheng &   Taoyuan &   Hsinchu &     Erlin &  Xinying &  Xiaogang &   Meinong \\
		& 0.051 &  -0.146 &  -0.032 &  -0.988 &  0.777 &   0.584 &   -0.125 &   0.201 \\
		\hline
		&&North&&&&East&&\\
		\hline
		N vs E &Guting &   Tucheng &   Taoyuan &  Hsinchu &     Yilan &  Dongshan &   Hualien &   Taitung \\
		&0.627 &   0.671 &   0.215 &   0.331 &  0.895 &   -0.051 &   0.441 &   0.044 \\
		\hline
		&&South&&&&East&&\\
		\hline
		S vs E&Erlin &   Xinying &  Xiaogang &   Meinong &     Yilan &  Dongshan &   Hualien &  Taitung \\
		&0.625 &   0.142 &    0.596 &   0.483 &  0.511 &   -0.145 &   0.303 &   0.791 \\
		\hline
	\end{tabularx}
	\caption{The loadings of monitoring stations of the first CCA.} \label{loadingStationsG}
\end{table}

\begin{table}[p]
	\centering
	\begin{tabularx}{\textwidth}{YYYYYYYY}
		\toprule
		&       SO2 &        CO &        O3 &      PM10 &       NOx &        NO &       NO2 \\
		\hline
		N &  0.001 &   0.478 &  -0.322 &  -0.132 &  -0.417 &   0.604 &  0.333 \\
		S &  0.571 &  -0.250 &   0.415 &   0.067 &   0.199 &  -0.618 &  0.110 \\
		\hline
		N &  -0.779 &   0.567 &  0.170 &  0.165 &  -0.051 &  0.112 &  -0.014 \\
		S &  -0.027 &  -0.566 &  0.402 &  0.409 &   0.212 &  0.010 &  -0.552 \\
		\hline
		S &  -0.270 &  -0.699 &  0.146 &  0.022 &  0.367 &  -0.399 &  -0.351 \\
		E &  -0.187 &  -0.292 &  0.434 &  0.080 &  0.219 &  -0.029 &  -0.798 \\
		\hline
	\end{tabularx}
	\caption{The loading of pollutants of the first CCA.} \label{loadingPollutantG}
\end{table}

Table~\ref{corrG} shows that  the correlations of air pollutants are high between regions, but,
as expected, they are not a pure function of the distance between monitoring stations.
The eastern stations, on average, are closer to the southern stations than
the northern stations, but the correlation between east and south is smaller than that
between north  and south. This is likely to be caused by the Central Mountain Range in Taiwan
with its peaks located in the central and southern parts of Taiwan.
Furthermore, the loadings of stations and pollutants shown in Tables~\ref{loadingStationsG}
and \ref{loadingPollutantG} are also informative.  The loading coefficients essentially
reflect the distances between the stations. The farther apart the stations are,
the smaller the magnitudes of the loadings. For example, Taitung has a higher loading
between South vs  East than that between North vs East. A similar effect is also seen in
Yilan and Guting. The loadings also reflect the sources of air pollutants.
The magnitude of the coefficient of Erlin, which is surrounded by industrial zones and near
a thermal power plant, is higher than other stations. The loadings of the pollutants vary,
but those of CO are higher and those of PM10 are lower in general.

See Appendix~\ref{appendix:air-pollution} for analysis of meteorological factor.


\section{Conclusion}
\label{sec:conclusion}

In this paper, we extended 2DCCA to tensor-valued data and provided a deeper
understanding of 2DCCA and TCCA. In particular, we showed that HOPM converges to
the global optimum and stationary points under different model assumptions. An
error bound for inexact updating for large scale data is provided. The results
are also justified by simulations with different models and parameters. Real
datasets are analyzed to demonstrate the ability of making use of the low rank
structure and the computational effectiveness of the proposed TCCA. The results
are encouraging, showing superior performance and high potential of TCCA.

\section*{Supplementary Materials}

\begin{description}
\item[Appendix:] 
The supplemental files include the Appendix which gives all proofs of Lemmas, Propositions, and Theorems. (TCCA\_appendix.pdf)
\item[Python code:] 
The supplemental files for this article include Python files and Jupyter notebooks which can be used to replicate the simulation study included in the article. In particular, they can generate Figure~\ref{fig:p2DCCA:lambda1} and Figure~\ref{fig:p2DCCA}.
\end{description}

\section*{Acknowledgements}

The authors thank Chi-Chun Liu for assistance with processing gene expression
data and Su-Yun Huang for comments that greatly improved the manuscript. This
work is partially supported by the William S.~Fishman Faculty Research Fund at
the University of Chicago Booth School of Business. This work was completed in
part with resources supported by the University of Chicago Research Computing
Center.

\bibliographystyle{my-plainnat}
\bibpunct{(}{)}{,}{a}{,}{,}
\bibliography{TCCA,paper}

\newpage

\title{\textbf{Supplementary Material: ``Tensor Canonical Correlation Analysis with Convergence and Statistical Guarantee''}}
\maketitle

\newpage

\appendix

\section{Technical Proofs} \label{sec:appendix}

\subsection{Proof of Proposition~\ref{prop:optimum-1DCCA-2DCCA}}
\label{sec:prof-prop:optimum-1DCCA-2DCCA}

    Let $\mPhi_i = \mM_i \mD_i \mN_i^\top$ and $\mOmega_i = \mP_i \mC_i \mQ_i^\top$
    be the SVDs of $\mPhi_i, \mOmega_i$ for $i=1,2$.
    The p2DCCA model \eqref{eq:p2DCCA} implies
    \begin{equation}
    \begin{split}
    \mSigma_{XX} = \E{\Vector{(\mX)}\Vector{(\mX)}^\top} &= (\mPhi_2 \mPhi_2^\top) \otimes (\mPhi_1\mPhi_1^\top) =  \mM_2 \mD_2^2 \mM_2^\top \otimes \mM_1 \mD_1^2 \mM_1^\top, \\
    \mSigma_{YY} = \E{\Vector{(\mY)}\Vector{(\mY)}^\top} &= (\mOmega_2 \mOmega_2^\top) \otimes (\mOmega_1\mOmega_1^\top) =  \mP_2 \mC_2^2 \mP_2^\top \otimes \mP_1 \mC_1^2 \mP_1^\top, \\
    \mSigma_{XY} = \E{\Vector{(\mX)}\Vector{(\mY)}^\top} &= (\mPhi_2 \otimes \mPhi_1) \operatorname{diag}{(\theta_{11}, \dots, \theta_{kk})} (\mOmega_2^\top \otimes \mOmega_1^\top).
    \end{split}
    \end{equation}
    The population cross-covariance matrix of p2DCCA is
    \begin{multline*}
    \E{\Vector{(\mX)}\Vector{(\mX)}^\top}^{-1/2} \E{\Vector{(\mX)}\Vector{(\mY)}^\top} \E{\Vector{(\mY)}\Vector{(\mY)}^\top}^{-1/2} \\
    = (\mM_2 \mN_2^\top \otimes \mM_1 \mN_1^\top) \operatorname{diag}{(\theta_{11}, \dots, \theta_{kk})} ( \mQ_2 \mP_2^\top  \otimes \mQ_1 \mP_1^\top).
    \end{multline*}
    Let $\bar{U}_{i}(r) = \mM_i N_{i}(r), \bar{V}_{i}(r) = \mP_i Q_{i}(r)$ where $\mN_i^\top=(N_{i}(1), \dots, N_{i}(k))$ and $\mQ_i^\top=(Q_{i}(1), \dots, Q_{i}(k))$. The first singular vectors are
    \begin{align*}
    \bar{U}^\ast = \mM_2 N_{2}(1) \otimes \mM_1 N_{1}(1) := \bar{U}_{2}^\ast \otimes \bar{U}_{1}^\ast
    \quad\text{and}\quad
    \bar{V}^\ast = \mP_2 Q_{2}(1) \otimes \mP_1 Q_{1}(1) := \bar{V}_{2}^\ast \otimes \bar{V}_{1}^\ast .
    \end{align*}
    This implies that the first CCA components are
    \begin{align*}
    U^\ast &= (\mPhi_2 \mPhi_2^\top)^{-1/2} \bar{U}_{2}^\ast \otimes (\mPhi_1 \mPhi_1^\top)^{-1/2} \bar{U}_{2}^\ast = \mM_2
    \mD_2^{-1} N_{2}(1) \otimes \mM_1 \mD_1^{-1} N_{1}(1) := U_2^\ast \otimes U_1^\ast \\
    \intertext{and}
    V^\ast &= (\mOmega_2 \mOmega_2^\top)^{-1/2} \bar{V}_{2}^\ast \otimes (\mOmega_1 \mOmega_1^\top)^{-1/2} \bar{V}_{1}^\ast = \mP_2 \mC_2^{-1} Q_{2}(1) \otimes \mP_1 \mC_1^{-1} Q_{1}(1) := V_2^\ast \otimes V_1^\ast.
    \end{align*}
    The proof is completed since the optimum of 1DCCA is contained in the feasible space of 2DCCA, The optimums of 2DCCA and 1DCCA coincide.

\subsection{Proof of Theorem~\ref{thm:p-conv-p2DCCA}}\label{sec:prof-thm-p-conv-p2DCCA}

Let $\bar{U}_{k+1,i} = \mSigma_{XX,j}^{1/2} U_{k+1, i}$ and $\bar{V}_{k+1,i} = \mSigma_{YY,j}^{1/2} V_{k+1, i}$.
Then we can rewrite \eqref{eq:HOPM-sim} 
\begin{equation*} 
	\begin{split}
	\tilde{U}_{k+1,i} &= \mSigma_{XX,j}^{-1} \mSigma_{XY,j}  V_{k,i}, \ \ \ U_{k+1,i} = \tilde{U}_{k+1,i} / \sqrt{ \tilde{U}_{k+1,i}^\top \mSigma_{XX,j} \tilde{U}_{k+1,i} }, \\
	\tilde{V}_{k+1,i} &= \mSigma_{YY,j}^{-1} \mSigma_{XY,j}^\top  U_{k,i}, \ \ \ V_{k+1,i} = \tilde{V}_{k+1,i} / \sqrt{\tilde{V}_{k+1,i}^\top \mSigma_{YY,j} \tilde{V}_{k+1,i}},
	\end{split}
\end{equation*}
as
\begin{equation*}
	\begin{split}
	\bar{U}_{k+1,i} &= \mSigma_{XX,j}^{-1/2} \mSigma_{XY,j} \mSigma_{YY,j}^{-1/2} \bar{V}_{k,i} / \| \mSigma_{XX,j}^{-1/2} \mSigma_{XY,j} \mSigma_{YY,j}^{-1/2} \bar{V}_{k,i}  \|, \\
	\bar{V}_{k+1,i} &= \mSigma_{YY,j}^{-1/2} \mSigma_{XY,j}^\top \mSigma_{XX,j}^{-1/2} \bar{U}_{k,i} / \| \mSigma_{YY,j}^{-1/2} \mSigma_{XY,j}^\top \mSigma_{XX,j}^{-1/2} \bar{U}_{k,i} \|.
	\end{split}
\end{equation*}

Let $\bar{U}_{r}, \bar{V}_{r}$ be the $r$-th left and right singular vectors of $\mSigma_{XX}^{-1} \mSigma_{XY} \mSigma_{YY}^{-1}$ for $r=1,\dots, d_1 d_2$ such that $\|\bar{U}_r\|=1=\|\bar{V}_r\|$. Then $\lbrace \bar{U}_{1}, \dots, \bar{U}_{d_1 d_2} \rbrace$ is an orthogonal basis. Define $k_1=d_1$ and $k_2=d_2$. Following the proof in Proposition~\ref{prop:optimum-1DCCA-2DCCA}, we know there exist $\lbrace \bar{U}_{i,1}^\ast, \dots, \bar{U}_{i,k_i}^\ast \rbrace_{i=1}^2$, $\lbrace \bar{V}_{i,1}^\ast, \dots, \bar{V}_{i,k_i}^\ast \rbrace_{i=1}^2$, and $r_1, r_2$ such that
\begin{equation}
    \bar{U}_{r} = \bar{U}_{2,r_1}^\ast \otimes \bar{U}_{1,r_2}^\ast, \ \ \ 
    \bar{V}_{r} = \bar{V}_{2,r_1}^\ast \otimes \bar{V}_{1,r_2}^\ast.
\end{equation}
Note that $\lbrace \bar{U}_{i,1}^\ast, \dots, \bar{U}_{i,k_i}^\ast \rbrace$ and $\lbrace \bar{V}_{i,1}^\ast, \dots, \bar{V}_{i,k_i}^\ast \rbrace$ are two orthogonal bases. Given vectors $\bar{U}_i, \bar{V}_i$, we have the decomposition:
\begin{equation}
    \bar{U}_i := \sum_{r=1}^{k_i} \alpha_{i,r} \bar{U}_{i,r}^\ast ,\quad
\bar{V}_i := \sum_{r=1}^{k_i} \beta_{i,r} \bar{V}_{i,r}^\ast, 
\end{equation}
Denote $\Theta_{1,i}=(\theta_{i1},\dots,\theta_{ik})$ and $\Theta_{2,i}=(\theta_{1i},\dots,\theta_{ki})$ and
\begin{align*}
U_i := (\mPhi_i \mPhi_i^\top)^{-1/2} \bar{U}_i,\quad
V_i := (\mOmega_i \mOmega_i^\top)^{-1/2} \bar{V}_i.
\end{align*}

By the low rank structure \eqref{eq:p2DCCA}, only $k^2$ singular values of $\mSigma_{XX}^{-1} \mSigma_{XY} \mSigma_{YY}^{-1}$ are non-zero, which means,
W.L.O.G, it suffices to assume $\bar{U}_i \in \operatorname{span}\{\bar{U}_{i,1}, \dots, \bar{U}_{i,k}\}$ and $\bar{V}_i \in \operatorname{span}\{\bar{V}_{i,1}, \dots, \bar{V}_{i,k}\}$.
Define
\[
A_j := \mPhi_j U_j = (\alpha_{j,1}, \dots, \alpha_{j,k})^\top , \quad B_j := \mOmega_j U_j = (\beta_{j,1}, \dots, \beta_{j,k})^\top.
\]
Then p2DCCA model yields for $(i,j)=(1,2),(2,1)$
\begin{equation}
\begin{split}
\E{\mX_j U_i U_i^\top \mX_j^\top} &:= \mSigma_{XX,j} = \|\mPhi_j U_j \|^2 \mM_i \mD_i^2 \mM_i^\top = \mM_i \mD_i^2 \mM_i^\top,\\
\E{\mY_j V_i V_i^\top \mY_j^\top} &:= \mSigma_{YY,j} = \|\mOmega_j V_j \|^2 \mP_i \mC_i^2 \mP_i^\top = \mP_i \mC_i^2 \mP_i^\top, \\
\E{\mX_j U_i V_i^\top \mY_j^\top} &:= \mSigma_{XY,j} = \mPhi_i \E{\mZ \mPhi_j U_j V_j^\top \mOmega_j \mZ^\top} \mOmega_i^\top, \\
&= \mPhi_i \left[ \operatorname{diag}(A_j^\top \operatorname{diag}(\Theta_{j,1}) B_j, \dots, A_j^\top \operatorname{diag}(\Theta_{j,k}) B_j) \right] \mOmega_i,
\end{split}
\end{equation}
where
\begin{align*}
\mX_j =
\begin{cases}
\mX & \text{ if } j=1\\
\mX^\top & \text{ if } j=2
\end{cases}, \quad
\mY_j =
\begin{cases}
\mY & \text{ if } j=1\\
\mY^\top & \text{ if } j=2
\end{cases},
\end{align*}
and we have used Lemma~\ref{lemma:compute-cov}.

To keep notation simple, we define following updates:
\begin{equation} \label{eq:HOPM}
\begin{split}
\bar{U}_{\text{NEW},i} &= \mSigma_{XX,j}^{-1/2} \mSigma_{XY,j} \mSigma_{XX,j}^{-1/2} \bar{V}_i / \| \mSigma_{XX,j}^{-1/2} \mSigma_{XY,j} \mSigma_{XX,j}^{-1/2} \bar{V}_i \|, \\
\bar{V}_{\text{NEW},i} &= \mSigma_{YY,j}^{-1/2} \mSigma_{XY,j}^\top \mSigma_{XX,j}^{-1/2} \bar{U}_i / \| \mSigma_{YY,j}^{-1/2} \mSigma_{XY,j}^\top \mSigma_{XX,j}^{-1/2} \bar{U}_i \|.
\end{split}
\end{equation}
Since
\[
\mSigma_{XX,j}^{-1/2} \mSigma_{XY,j} \mSigma_{YY,j}^{-1/2}
= \mM_i \mN_i^\top \left[ \operatorname{diag}(A_j^\top \operatorname{diag}(\Theta_{j,1}) B_j, \dots, A_j^\top \operatorname{diag}(\Theta_{j,k}) B_j) \right] \mQ_i \mP_i^\top,
\]
and
\begin{align*}
|A_j^\top \operatorname{diag}(\Theta_{j,1}) B_j|
\geq& |\alpha_{j1} \theta_{11} \beta_{j1}| - \left| \sum_{r=2}^k \alpha_{jr} \theta_{1r} \beta_{jr}\right | \\
\geq& |\alpha_{j1} \theta_{11} \beta_{j1}| - \theta_{12}\sqrt{(1-\alpha_{j1}^2)(1-\beta_{j1}^2)} \\
|A_j^\top \operatorname{diag}(\Theta_{j,r}) B_j| \leq& \theta_{12}, \quad r=2,\dots,k,
\end{align*}
by the standard argument of the power method (See Theorem 8.2.1 in \cite{golub2012matrix}), we know for $(i,j)=(1,2),(2,1)$
\begin{align*}
| \sin \theta(\bar{U}_{i}^\ast, \bar{U}_{\text{NEW},i}) | \leq \tan \theta(\bar{V}_{i}^\ast, \bar{V}_i) \frac{\theta_{11} |\alpha_{j1} \beta_{j1}| - \theta_{12}\sqrt{(1-\alpha_{j1}^2)(1-\beta_{j1}^2)}}{\theta_{12}}, \\
| \sin \theta(\bar{V}_{i}^\ast, \bar{V}_{\text{NEW},i}) | \leq \tan \theta(\bar{U}_{i}^\ast, \bar{U}_i) \frac{\theta_{11}|\alpha_{j1} \beta_{j1}| - \theta_{12}\sqrt{(1-\alpha_{j1}^2)(1-\beta_{j1}^2)}}{\theta_{12}},
\end{align*}
provided that
$\theta_{12}< \theta_{11} |\alpha_{j1} \beta_{j1}| - \theta_{12}\sqrt{(1-\alpha_{j1}^2)(1-\beta_{j1}^2)}$,
where $\cos \theta(W_1, W_2) = W_1^\top W_2 / (\|W_1\| \|W_2 \|)$. 
We also have
\begin{align*}
|\cos \theta(\bar{U}_{i}^\ast, \bar{U}_i)| = |\alpha_{i1}|, \ \ \ |\cos \theta(\bar{V}_{i}^\ast, \bar{V}_i)| = |\beta_{i1}|.
\end{align*}
and
$\sin \theta(\bar{U}_{i}^\ast, \bar{U}_{k, i}) = \sin_{X, j} \theta(U_{i}^\ast, U_{k, i})$, $\sin \theta(\bar{V}_{i}^\ast, \bar{V}_{k, i}) = \sin_{Y, j} \theta(V_{i}^\ast, V_{k, i})$
where for $\square = X, Y$
\[
\cos_{\square, j} \theta(W_1, W_2) = W_1^\top \mSigma_{\square\square,j} W_2 / \sqrt{ W_1^\top \mSigma_{\square\square,j} W_1 W_2^\top \mSigma_{\square\square,j} W_2} 
\] 
This completes the proof.

\subsection{Proof of Theorem~\ref{thm:p-conv-p2DCCA-sample}}
\label{sec:p-conv-p2DCCA-sample}

Note that we assume almost surely
\[
\max\{\|\mX\|, \|\mX^\top\|, \|\mY\|, \|\mY^\top\|\}\leq 1.
\]
Assume $\{\mX_t \mY_t\}_{t=1}^n$ i.i.d. sampled from (\ref{eq:p2DCCA}).
Define the matrices:
\begin{align*}
\mT_j(\epsilon_n) = (\mSigma_{XX,j}+\epsilon_n \mI)^{-1/2} \mSigma_{XY,j} (\mSigma_{XX,j}+\epsilon_n \mI)^{-1/2}, \\
\hat{\mT_j}(\epsilon_n) = (\hat{\mSigma}_{XX,j}+\epsilon_n \mI)^{-1/2} \hat{\mSigma}_{XY,j} (\hat{\mSigma}_{XX,j}+\epsilon_n \mI)^{-1/2},
\end{align*}
where
\begin{align*}
\mX_{j,t} =
\begin{cases}
\mX_t & \text{ if } j=1\\
\mX_t^\top & \text{ if } j=2
\end{cases}, \quad
\mY_{j,t} =
\begin{cases}
\mY_t & \text{ if } j=1\\
\mY_t^\top & \text{ if } j=2
\end{cases}.
\end{align*}
Then \eqref{eq:HOPM} can be rewritten using $\mT_j(\epsilon_n)$
Consider the following updating
\begin{equation*} 
\begin{split}
\bar{U}_{\text{NEW},i} &= \hat{\mT_j}(\epsilon_n) \bar{V}_i / \| \hat{\mT_j}(\epsilon_n) \bar{V}_i \|,
\\
\bar{V}_{\text{NEW},i} &= \hat{\mT_j}(\epsilon_n)^\top \bar{U}_i / \| \hat{\mT_j}(\epsilon_n)^\top \bar{U}_i \|.
\end{split}
\end{equation*}
We show the sample version of convergence theorem.
This can be analyzed by Lemma~\ref{lemma:nosiy-PM} of one step of noisy power method.
We modify the proof from Lemma 2.3 in \cite{hardt2014noisy} to our setting.
Note that it is possible to adapt to a finer bound provided in \cite{balcan2016improved}.

Since
$\hat{\mT_j}(\epsilon_n) = \mT_j(0) + (\mT_j(\epsilon_n)-\mT_j(0)) + (\mT_j(\epsilon_n)-\hat{\mT_j}(\epsilon_n))$, we only need to bound the noisy term. To do this, 
we first introduce some useful lemmas:

By Lemma~\ref{lemma:matrix-hoeffding}, we have
\begin{align*}
\operatorname{Prob} (\|\mSigma_{XX,j} - \hat{\mSigma}_{XX,j}\| \geq& \delta_n) \leq \max \{d_1,d_2\} \exp{(- \delta_n^2 n /8)}, \\
\operatorname{Prob} (\|\mSigma_{YY,j} - \hat{\mSigma}_{YY,j}\| \geq& \delta_n) \leq \max \{d_1,d_2\} \exp{(- \delta_n^2 n /8)}, \\
\operatorname{Prob} (\|\mSigma_{XY,j} - \hat{\mSigma}_{XY,j}\| \geq& \delta_n) \leq 2 \max \{d_1,d_2\} \exp{(- \delta_n^2 n /8)}.
\end{align*}

It is easy to see that given $\mA, \mB, \mC, \hat{\mA}, \hat{\mB}, \hat{\mC}$, we have
\begin{enumerate}
	\item $\mA^{-1/2} - \mB^{-1/2} = \mA^{-1/2}(\mB^{3/2}-\mA^{3/2})\mB^{-3/2}+(\mA-\mB)\mB^{-3/2}$;
	\item $ \mA \mB \mC - \hat{\mA} \hat{\mB} \hat{\mC} = (\mA - \hat{\mA})  \hat{\mB} \hat{\mC}.
	+  \mA (\mB - \hat{\mB}) \hat{\mC}
	+  \mA \mB (\mC - \hat{\mC}) $.
\end{enumerate}
Combining above equations and Lemma~\ref{lemma:AB32} implies
\begin{align*}
\| \mT_j(\epsilon_n)-\mT_j(0)) \| =&
\| \mM_i (D_i+\epsilon_n \mI)^{-1}D_i \mN_i^\top \E{\mZ \mPhi_j U_j V_j^\top \mOmega_j \mZ^\top} \mQ_i (C_i+\epsilon_n \mI)^{-1}C_i \mP_i^\top \\
&- \mM_i \mN_i^\top \E{\mZ \mPhi_j U_j V_j^\top \mOmega_j \mZ^\top} \mQ_i \mP_i^\top \| \\
\leq & O(\epsilon_n).
\end{align*}
and with probability $1-\max \{d_1,d_2\} \exp{(- \delta_n^2 n /8)}$
\begin{align*}
& ((\mSigma_{XX,j}+\epsilon_n \mI)^{-1/2}-(\hat{\mSigma}_{XX,j}+\epsilon_n \mI)^{-1/2}) \mSigma_{XY,j} (\mSigma_{XX,j}+\epsilon_n \mI)^{-1/2} \\
=& \left[ (\mSigma_{XX,j}+\epsilon_n \mI)^{-1/2} ((\hat{\mSigma}_{XX,j}+\epsilon_n \mI)^{3/2} - (\mSigma_{XX,j}+\epsilon_n \mI)^{3/2}) \right. \\
&\left. +(\mSigma_{XX,j}-\hat{\mSigma}_{XX,j})  \right] (\mSigma_{XX,j}+\epsilon_n \mI)^{-3/2}\mSigma_{XY,j} (\mSigma_{XX,j}+\epsilon_n \mI)^{-1/2} \\
=& O(\epsilon_n^{-3/2} \delta_n),
\end{align*}
where we use $\max \{\|\mT_j(\epsilon_n) \|, \|\hat{\mT_j}(\epsilon_n)\|\} = O(1)$.
Thus, with probability $1- 3\max \{d_1,d_2\} \exp{(- \delta_n^2 n /8)}$ we have
\begin{align*}
\|\mT_j(\epsilon_n)-\hat{\mT_j}(\epsilon_n)\| =& \| (\mSigma_{XX,j}+\epsilon_n \mI)^{-1/2} \mSigma_{XY,j} (\mSigma_{XX,j}+\epsilon_n \mI)^{-1/2} \\
&- (\hat{\mSigma}_{XX,j}+\epsilon_n \mI)^{-1/2} \hat{\mSigma}_{XY,j} (\hat{\mSigma}_{XX,j}+\epsilon_n \mI)^{-1/2} \| \\
=& O(\epsilon_n^{-3/2} \delta_n).
\end{align*}
Thus, by Lemma~\ref{lemma:nosiy-PM} and choosing $\delta = c_1/\sqrt{n}$, if $\epsilon_n + \epsilon_n^{-3/2} n^{1/2} \leq c$,
where $c_1, c_2$ are constants depending on initialization and $c$ is a small constant depending on initialization, $c_1$, $\varepsilon$,$\theta_{11}, \theta_{12}$, and $\mPhi_1, \mPhi_2, \mOmega_1, \mOmega_2$, we have 
\begin{equation*}
\begin{split}
    \tan \theta(\bar{U}_i^\ast, \bar{U}_{\text{NEW}, i}) \leq \max \left( \epsilon, \max \left( \epsilon, (\sigma_{2}/\sigma_1)^{1/4} \right) \tan \theta(\bar{V}_i^\ast, \bar{V}_{i}) \right), \\
    \tan \theta(\bar{V}_i^\ast, \bar{V}_{\text{NEW}, i}) \leq \max \left( \epsilon, \max \left( \epsilon, (\sigma_{2}/\sigma_1)^{1/4} \right) \tan \theta(\bar{U}_i^\ast, \bar{U}_{i}) \right).
\end{split}
\end{equation*}
This implying with probability $1-\max \{d_1,d_2\} \exp{(- c_1 /8)}$, we have
	\[
	    \max_i \left\{ |\tan \theta(\bar{U}_{i}^\ast, \bar{U}_{T,i})|,  |\tan \theta(\bar{V}_{i}^\ast, \bar{V}_{T,i}))| \right\} \leq \epsilon,
	\]
for $T \geq c_2 (1-\theta_{12}/\tilde{\theta}_{11})^{-1} \log(1/\epsilon)$.
Note that $\bar{U}_{k+1,i} = (\mSigma_{XX,j}+\epsilon_n \mI)^{1/2} U_{k+1, i}$ and $\bar{V}_{k+1,i} = (\mSigma_{YY,j}+\epsilon_n \mI)^{1/2} V_{k+1, i}$.
This yields our finite sample bound:
\begin{equation*}
    \max_{(i,j)} \left\{ |\widehat{\sin}_{X,j} \theta(U_{i}^\ast, U_{T,i})|,  |\widehat{\sin}_{Y, j} \theta(V_{i}^\ast, V _{T,i}))| \right\} \leq \epsilon,
\end{equation*}
where
\[
\widehat{\cos}_{\square, j} \theta(W_1, W_2) = W_1^\top (\hat{\mSigma}_{\square\square,j}+\epsilon_n \mI) W_2 / \sqrt{ W_1^\top (\hat{\mSigma}_{\square\square,j}+\epsilon_n \mI) W_1 W_2^\top (\hat{\mSigma}_{\square\square,j}+\epsilon_n \mI)W_2} 
\] 
and $\widehat{\sin}_{\square, j}^2 = 1- \widehat{\cos}_{\square, j}^2$ for $\square = X,Y$.

\subsection{Proof of Proposition~\ref{prop:just-sHOPM}}
\label{sec:prof-prop-just-sHOPM}

    First compute the gradient of the potential \eqref{eq:mloss}:
    \begin{equation}
    \begin{split}
    \nabla_{U_j} \tilde{\cL} & = \frac{ \alpha^2 (1-2\lambda)}{n} \mX_j^\top \mX_j U_j -\frac{\alpha \beta}{n} \mX_j^\top \mY_j V_j, \\
    \nabla_{\alpha} \tilde{\cL} & = \frac{\alpha (1-2\lambda)}{n} U_j^\top \mX_j^\top \mX_j U_j - \frac{\beta}{n} U_j^\top \mX_j^\top \mY_j V_j, \\
    \nabla_{\lambda} \tilde{\cL} & = 1-\frac{\alpha^2}{n} U_j^\top \mX_j^\top \mX_j U_j,
    \end{split}
    \end{equation}
    and, similarly for $\nabla_{\beta,V_j,\mu} \tilde{\cL}$.
    Then by plugging \eqref{ALSdynamic} into \eqref{eq:mloss}
    and getting
    \begin{align*}
    \nabla_{\alpha,U_j,\lambda} \tilde{\cL} (\alpha_{kj}, \cU_{kj}, \lambda_{kj}, \beta_{k,j}, \cV_{k-1,j}, \mu_{k-1,j}) = 0,  \\
    \nabla_{\beta,V_j,\mu} \tilde{\cL} (\alpha_{kj}, \cU_{kj}, \lambda_{kj}, \beta_{kj}, \cV_{kj}, \mu_{kj}) = 0.
    \end{align*}
    the proposition follows

\subsection{Proof of Proposition~\ref{prop:hopm-als}}
\label{sec:prof-prop}

Our goal here is to show the connection between two regularizations.
It suffices to show that there exist $a_{kj}>0, b_{kj} >0$, for all $k, j$, such that
\[
U_{kj} = a_{kj} A_{kj},\quad  V_{kj} = b_{kj} B_{kj}.
\]
We show this by induction. Since both algorithms start at the same point,
the result holds for $k=0$, $j=1,\ldots,m$.
By the hypothesis and construction of $A_{kj}$, we have
\begin{align*}
U_{kj} & = \frac{\mX_{kj}^\dagger \mY_{kj} V_{k-1,j}}{\sqrt{V_{k-1,j}^\top  \mY_{kj}^\top \mX_{kj} \mX_{kj}^\dagger \mY_{kj} V_{k-1,j}}} \\
& = \frac{b_{k-1,j} \mX_{kj}^\dagger \mY_{kj} B_{k-1,j}}{\sqrt{V_{k-1,j}^\top  \mY_{kj}^\top \mX_{kj} \mX_{kj}^\dagger \mY_{kj} V_{k-1,j}}} \\
& = \frac{b_{k-1,j} \| \mX_{kj}^\dagger \mY_{kj} B_{k-1,j}\| A_{kj}}{\sqrt{V_{k-1,j}^\top  \mY_{kj}^\top \mX_{kj} \mX_{kj}^\dagger \mY_{kj} V_{k-1,j}}}.
\end{align*}
Similar argument holds for $V_{kj}$. Because all constants are positive and correlation is scale-invariant,
this yields the result. Furthermore, by construction, we have
\[
1 = U_{kj}^\top \mX_{kj}^\top \mX_{kj} U_{kj} = a_{kj}^2 A_{kj}^\top \mX_{kj}^\top \mX_{kj} A_{kj}.
\]
Following the same argument for $V_{kj} = B_{kj} / \sqrt{B_{kj}^\top (\frac{1}{n}\mY_{kj}^\top \mY_{kj}) B_{kj}}$, we complete the proof of the first claim.

For the second part, we have
\begin{align*}
\alpha^{(1/m)} \nabla_{U_j} \cL(\alpha^{(1/m)}A_1, &\dots, \alpha^{(1/m)}A_m, \beta^{(1/m)}\lambda, B_1, \dots, \beta^{(1/m)}B_m, \mu) \\
= &  \frac{ \alpha^2 (1-2\lambda)}{n} \mX_j^\top \mX_j A_j -\frac{\alpha \beta}{n} \mX_j^\top \mY_j B_j \\
= & \nabla_{U_j} \tilde{\cL} (\alpha, A_1, \dots, A_m, \lambda, \beta, B_1, \dots, B_m, \mu) \\
= & 0,
\end{align*}
and
\begin{align*}
\nabla_{\lambda} \cL(\alpha^{(1/m)}A_1, &\dots, \alpha^{(1/m)}A_m, \beta^{(1/m)}\lambda, B_1, \dots, \beta^{(1/m)}B_m, \mu) \\
= & 1 -\frac{ \alpha^2 (1-2\lambda)}{n} \mX_j^\top \mX_j A_j  \\
= & \nabla_{\lambda} \tilde{\cL} (\alpha, A_1, \dots, A_m, \lambda, \beta, B_1, \dots, B_m, \mu) \\
= & 0.
\end{align*}
Since $\alpha^{(1/m)}>0$, $\nabla_{U_j} \cL(\alpha^{(1/m)}A_1, \dots, \alpha^{(1/m)}A_m, \beta^{(1/m)}\lambda, B_1, \dots, \beta^{(1/m)}B_m, \mu)=0$. Applying the same argument for $\nabla_{\beta} \cL$ and $\nabla_{\mu} \cL$, we obtain the proposition.

\subsection{Proof of Theorem \ref{localConv}}

We are prove Theorem~\ref{localConv} in this section.
It is clear that $\tilde{\cL}$ is analytic,
so it suffices to verify the three conditions in Lemma~\ref{convthm}.

	\paragraph{Stationary Condition.}
	Since other variables only depend on $U_{k+1,j}$ and $V_{k+1,j}$, it suffices to show that $U_{kj} = U_{k+1,j}$ and $V_{kj} = V_{k+1,j}$. By a symmetric argument, we only show the part for $U_{k+1,j}$. Note that $\nabla_{U_j} \tilde{\cL}(\alpha_{kj}, \cU_{k+1,j}) = 0$ implies $\alpha_{kj}(1-\lambda_{kj}) \mX_{kj}^\top \mX_{kj} U_{kj}= \mX_{kj}^\top \mY_{kj} V_{k-1,j}$, so we have $\tilde{U}_{k+1,j} = \alpha_{kj}(1-\lambda_{kj}) U_{kj}$. After normalization, we obtain $U_{k,j}=U_{k+1,j}$, and the stationary condition follows.

	\paragraph{Asymptotic small step-size safeguard.}
	It suffices to show the part for $U_{kj}$.
	Since $(1-2\lambda_{k,j}) = \rho_n(\cU_{kj}, \cV_{kj})$ is bounded and by Statement 2 in Lemma~\ref{lemma:technical}, $(\alpha_{kj}, U_{kj}, \lambda_{kj})$ is on a compact set.
	Combining compactness and
	\begin{align*}
	\nabla_{\alpha,U_j,\lambda} \tilde{\cL} (\alpha_{k+1,j}, \cU_{k+1,j}, \lambda_{k+1,j}, \beta_{k+1,j-1}, \cV_{k+1,j-1}, \mu_{k+1,j-1}) = 0,  \\
	\nabla_{\beta,V_j,\mu} \tilde{\cL} (\alpha_{k+1,j}, \cU_{k+1,j}, \lambda_{k+1,j}, \beta_{k+1,j}, \cV_{k+1,j}, \mu_{k+1,j}) = 0,
	\end{align*}
	we deduce that, for some $L>0$ independent on $k,j$,
	\begin{align*}
	& \| \nabla \tilde{\cL} (\alpha_{k0}, \cU_{k0}, \lambda_{k0}, \beta_{k0}, \cV_{k0}, \mu_{k0}) \|^2 \\
	= & \sum_{j} \|  \nabla_{U_j} \tilde{\cL} (\alpha_{k0}, \cU_{k0}, \lambda_{k0}, \beta_{k0}, \cV_{k0}, \mu_{k0}) - \nabla_{U_j} \tilde{\cL} (\alpha_{k+1,j}, \cU_{k+1,j}, \lambda_{k+1,j}, \beta_{k+1,j-1}, \cV_{k+1,j-1}, \mu_{k+1,j-1}) \|^2 \\
	+ &  \| \nabla_{\alpha,\lambda} \tilde{\cL} (\alpha_{k0}, \cU_{k0}, \lambda_{k0}, \beta_{k0}, \cV_{k0}, \mu_{k0}) - \nabla_{\alpha,\lambda} \tilde{\cL} (\alpha_{km}, \cU_{km}, \lambda_{km}, \beta_{k,m-1}, \cV_{k,m-1}, \mu_{k,m-1})\|^2 \\
	+ & \sum_{j} \|  \nabla_{V_j} \tilde{\cL} (\alpha_{k0}, \cU_{k0}, \lambda_{k0}, \beta_{k0}, \cV_{k0}, \mu_{k0}) - \nabla_{V_j} \nabla \tilde{\cL} (\alpha_{k+1,j}, \cU_{k+1,j}, \lambda_{k+1,j}, \beta_{k+1,j}, \cV_{k+1,j}, \mu_{k+1,j}) \|^2 \\
	+ &  \| \nabla_{\beta,\mu} \tilde{\cL} (\alpha_{k0}, \cU_{k0}, \lambda_{k0}, \beta_{k0}, \cV_{k0}, \mu_{k0}) - \nabla_{\beta,\mu} \tilde{\cL} (\alpha_{km}, \cU_{km}, \lambda_{km}, \beta_{k,m}, \cV_{k,m}, \mu_{k,m})\|^2 \\
	\leq & L^2(2m+4) \| (\alpha_{k0}, \cU_{k0}, \lambda_{k0}, \beta_{k0}, \cV_{k0}, \mu_{k0}) -  (\alpha_{km}, \cU_{km}, \lambda_{km}, \beta_{k,m}, \cV_{k,m}, \mu_{k,m} \|^2.
	\end{align*}
	This completes the asymptotic small step-size safeguard condition.

	\paragraph{Primary descent condition.}
	We use the fact that updating each component is a least squares problem to prove the following
	\begin{equation*}
	\begin{split}
	\tilde{\cL}&(\alpha_{k0}, \cU_{k+1,0}, \lambda_{k,0}, \beta_{k+1,0}, \cV_{k+1,0}, \mu_{k+1,0}) - \tilde{\cL}(\alpha_{k+1,m}, \cU_{k+1,m}, \lambda_{k+1,m}, \beta_{k+1,m}, \cV_{k+1,m}, \mu_{k+1,m}) \\
	& \geq \sum_{j=1}^{m}\left[  \tilde{\cL}(\alpha_{k,j-1}, \cU_{k+1,j-1}, \lambda_{k,j-1}, \beta_{k,j}, \cV_{k+1,j-1}, \mu_{k,j})
	- \tilde{\cL}(\alpha_{k+1,j}, \cU_{k+1,j}, \lambda_{k+1,j}, \beta_{k,j}, \cV_{k+1,j-1}, \mu_{k,j}) \right. \\
	& \quad \qquad
	\left.+\tilde{\cL}(\alpha_{k+1,j}, \cU_{k+1,j}, \lambda_{k+1,j}, \beta_{k,j}, \cV_{k+1,j-1}, \mu_{k,j})
	- \tilde{\cL}(\alpha_{k+1,j}, \cU_{k+1,j}, \lambda_{k+1,j}, \beta_{k+1,j}, \cV_{k+1,j}, \mu_{k+1,j}) \right] \\
	& \geq  \frac{\sigma_0}{2} \left[  \sum_{j=1}^{m} (\alpha_{k,j}-\alpha_{k+1,j})^2 + ( \lambda_{k,j} -  \lambda_{k+1,j})^2 + \| U_{k+1,j}-U_{kj} \|^2 \right.  \\
	& \quad\qquad \left.  + \sum_{j=1}^{m} (\beta_{k,j}-\beta_{k+1,j})^2 + ( \mu_{k,j} -  \mu_{k+1,j})^2+ \| V_{k+1,j}-V_{kj} \|^2 \right] \\
	& \geq \frac{\sigma_0}{2m} \left[  (\alpha_{k,0}-\alpha_{k+1,m})^2 + ( \lambda_{k,0} -  \lambda_{k+1,m})^2 + \| \cU_{k+1,0}-\cU_{k+1,m} \|^2 \right.  \\
	& \quad\qquad \left.  + (\beta_{k,0}-\beta_{k+1,m})^2 + ( \mu_{k,0} -  \mu_{k+1,m})^2+ \| \cV_{k+1,j}-\cV_{kj} \|^2 \right],
	\end{split}
	\end{equation*}
	where  the last inequality is from the fact that $\|W_m-W_0\|^2 \leq m \sum_{j=1}^m \|W_i-W_{j-1}\|^2$.
	Combining this and the asymptotic small step-size safeguard yields the primary descent condition.

\subsection{Proof of Theorem \ref{thm:inexactupdating}}

We show the error bound of inexact updatingu in this section.
We only focus on the $U_{kj}$ since a similar argument directly applies to $V_{kj}$.

In this proof, we distinguish the iterates of inexact updating power iterations
(inexact updating in Algorithm \ref{alg:TCCA}) from the iterates of the exact
power iterations (exact updating in Algorithm \ref{alg:TCCA}) and denote the latter with
asterisks, i.e., $U_{kj}$ and $U_{kj}^\ast$. Let
$f_{kj}(\tilde{U}) = \frac{1}{2n} \| \mX_{kj} \tilde{U} - \mY_{kj} V_{k-1,j} \|^2$ and
$g_{kj}(\tilde{V}) = \frac{1}{2n} \| \mX_{kj} \tilde{U}_{kj} - \mY_{kj} \tilde{V} \|^2$.
We denote the exact optimum of $f_{kj}(U)$ and $g_{kj}(V)$ by $\tilde{U}_{kj}^\natural$
and $\tilde{V}_{kj}^\natural$ respectively, and use tilde to indicate that the iterates are unnormalized, i.e., $\tilde{U}_{kj}$ and $\tilde{U}_{kj}^\ast$.

We prove the theorem by induction, exploiting the recurrent relationship of the error bound.
By the triangle inequality, we have
\begin{equation} \label{eq:errorbound}
   \| U_{kj} - U_{kj}^{\ast} \| \leq \| U_{kj} - U_{kj}^\natural \| + \| U_{kj} - U_{kj}^{\ast} \| = \mbox{(I)} + \mbox{(II)}.
\end{equation}
For the first term, by construction and the fact that $\tilde{U}_{kj}$ is an $\epsilon$-suboptimum of $f_{kj}$, we have
\[
\epsilon \geq f_{kj}(\tilde{U}_{kj}) - f_{kj}(\tilde{U}_{kj}^\natural) = \frac{1}{2} (\tilde{U}_{kj} - \tilde{U}_{kj}^\natural)^\top \mX_{kj}^\top \mX_{kj} (\tilde{U}_{kj} - \tilde{U}_{kj}^\natural) \geq \sigma_{l,x} \| \tilde{U}_{kj} - \tilde{U}_{kj}^\natural \|^2.
\]
For the (I) in \eqref{eq:errorbound}, Lemma~\ref{lemma:boundedsolution} implies that
$\| \tilde{U}_{kj}^\natural \|$ is uniformly bounded below for all $k,j$, yielding that for some $c>0$
\[
\| U_kj - U_kj^\natural \| \leq \mbox{tan}^{-1} \left(  \frac{\| \tilde{U}_{kj} - \tilde{U}_{kj}^\natural \|}{\|\tilde{U}_{kj}^\natural \|} \right) \leq c \sigma_{l,x}^{-1} \sqrt{\epsilon}.
\]
For the (II) in \eqref{eq:errorbound}, agian by construction, we have
\begin{align*}
\| U_{kj} - U_{kj}^{\ast} \| & = \| (\mX_{k,j}^\top \mX_{k,j})^{-1} \mX_{k,j}^\top \mY_{k,j} V_{k-1,j} - ((\mX_{k,j}^\ast)^\top \mX_{k,j}^\ast)^{-1} (\mX_{k,j}^\ast)^\top \mY_{k,j}^\ast V_{k-1,j}^\ast \| \\
& \leq \| (\mX_{k,j}^\top \mX_{k,j})^{-1} \mX_{k,j}^\top \mY_{k,j} V_{k-1,j} - (\mX_{k,j}^\top \mX_{k,j})^{-1} \mX_{k,j}^\top \mY_{k,j} V_{k-1,j}^\ast \| \\
& + \| (\mX_{k,j}^\top \mX_{k,j})^{-1} \mX_{k,j}^\top \mY_{k,j} V_{k-1,j}^\ast - (\mX_{k,j}^\top \mX_{k,j})^{-1} \mX_{k,j}^\top \mY_{k,j}^\ast V_{k-1,j}^\ast \| \\
& + \| (\mX_{k,j}^\top \mX_{k,j})^{-1} \mX_{k,j}^\top \mY_{k,j}^\ast V_{k-1,j}^\ast - (\mX_{k,j}^\top \mX_{k,j})^{-1} (\mX_{k,j}^\ast)^\top \mY_{k,j}^\ast V_{k-1,j}^\ast \| \\
& + \| (\mX_{k,j}^\top \mX_{k,j})^{-1} (\mX_{k,j}^\ast)^\top \mY_{k,j}^\ast V_{k-1,j}^\ast - ((\mX_{k,j}^\ast)^\top \mX_{k,j}^\ast)^{-1} (\mX_{k,j}^\ast)^\top \mY_{k,j}^\ast V_{k-1,j}^\ast \|,
\end{align*}
where $\mX_{k,j}^\ast$ is the exact version of $\mX_{k,j}$. Applying the same technique,
we obtain
\begin{align*}
 \| (\mX_{k,j}^\top \mX_{k,j})^{-1} \mX_{k,j}^\top \mY_{k,j} V_{k-1,j} &- (\mX_{k,j}^\top \mX_{k,j})^{-1} \mX_{k,j}^\top \mY_{k,j} V_{k-1,j}^\ast \| \\
& \leq \| \frac{1}{n}(\mX_{k,j}^\top \mX_{k,j})^{-1} \| \| \frac{1}{\sqrt{n}} \mX_{k,j}^\top \| \| \frac{1}{\sqrt{n}} \mY_{k,j} \| \|V_{k-1,j}^\ast - V_{k-1,j}^\ast \| \\
& \leq \sigma_{l,x}^{-1} \sigma_{u,x}^{1/2} \sigma_{u,y}^{1/2}  \|V_{k-1,j}^\ast - V_{k-1,j}^\ast \|
\end{align*}
and
\begin{align*}
\| (\mX_{k,j}^\top \mX_{k,j})^{-1} \mX_{k,j}^\top \mY_{k,j} V_{k-1,j}^\ast &- (\mX_{k,j}^\top \mX_{k,j})^{-1} \mX_{k,j}^\top \mY_{k,j}^\ast V_{k-1,j}^\ast \| \\
& \leq \| \frac{1}{n}(\mX_{k,j}^\top \mX_{k,j})^{-1} \| \| \frac{1}{\sqrt{n}} \mX_{k,j}^\top \| \| \frac{1}{\sqrt{n}} (\mY_{k,j} - \mY_{k,j}^\ast) \| \\
& \leq \sigma_{l,x}^{-1} \sigma_{u,x}^{1/2} \sigma_{u,y}^{1/2} \left( \sum_{i<j} \| V_{k,i} - V_{ki}^\ast \| + \sum_{i>j} \| V_{k-1,i} - V_{k-1,i}^\ast \| \right).
\end{align*}
By a similar argument,
\begin{align*}
\| (\mX_{k,j}^\top \mX_{k,j})^{-1} \mX_{k,j}^\top \mY_{k,j}^\ast V_{k-1,j}^\ast &- (\mX_{k,j}^\top \mX_{k,j})^{-1} (\mX_{k,j}^\ast)^\top \mY_{k,j}^\ast V_{k-1,j}^\ast \| \\
& \leq \sigma_{l,x}^{-1} \sigma_{u,x}^{1/2} \sigma_{u,y}^{1/2} \left( \sum_{i<j} \| U_{k,i} - U_{ki}^\ast \| + \sum_{i>j} \| U_{k-1,i} - U_{k-1,i}^\ast \| \right),   \\
\| (\mX_{k,j}^\top \mX_{k,j})^{-1} (\mX_{k,j}^\ast)^\top \mY_{k,j}^\ast V_{k-1,j}^\ast &- ((\mX_{k,j}^\ast)^\top \mX_{k,j}^\ast)^{-1} (\mX_{k,j}^\ast)^\top \mY_{k,j}^\ast V_{k-1,j}^\ast \| \\
&\leq \sigma_{l,x}^{-1} \sigma_{u,x}^{1/2} \sigma_{u,y}^{1/2} \left( \sum_{i<j} \| U_{k,i} - U_{ki}^\ast \| + \sum_{i>j} \| U_{k-1,i} - U_{k-1,i}^\ast \| \right).
\end{align*}
By induction of hypothesis, we have, for some $c>0$ (the constant may change from line to line)
\begin{multline*}
\| U_{kj} - U_{kj}^{\ast} \|
\leq c \sigma_{l,x}^{-1} \sqrt{\epsilon} + 2 \sigma_{l,x}^{-1} \sigma_{u,x}^{1/2} \sigma_{u,y}^{1/2} \left( \sum_{i<j} \| U_{k,i} - U_{ki}^\ast \| + \sum_{i>j} \| U_{k-1,i} - U_{k-1,i}^\ast \|\right.  \\
+  \left.  \sum_{i<j} \| V_{k,i} - V_{ki}^\ast \| + \sum_{i>j} \| V_{k-1,i} - V_{k-1,i}^\ast \| \right).
\end{multline*}
Similarly, we have
\begin{multline*}
\| V_{kj} - V_{kj}^{\ast} \| \leq c \sigma_{l,y}^{-1} \sqrt{\epsilon} + 2 \sigma_{l,y}^{-1} \sigma_{u,y}^{1/2} \sigma_{u,x}^{1/2} \left(  \sum_{i<j} \| V_{k,i} - V_{ki}^\ast \| + \sum_{i>j} \| V_{k-1,i} - V_{k-1,i}^\ast \|  \right.  \\
+ \left. \sum_{i<j} \| U_{k,i} - U_{ki}^\ast \| + \sum_{i>j} \| U_{k-1,i} - U_{k-1,i}^\ast \| \right).
\end{multline*}
Define $c_1=\max\{c \sigma_{l,y}^{-1}, c \sigma_{l,x}^{-1} \}$,
$c_2=\max\{2 \sigma_{l,y}^{-1} \sigma_{u,y}^{1/2} \sigma_{u,x}^{1/2}, 2 \sigma_{l,x}^{-1} \sigma_{u,x}^{1/2} \sigma_{u,y}^{1/2} \}$,
\begin{equation*}
E_\ell  =  \begin{cases}
0 & \text{if $\ell\leq0$}, \\
\| U_{kj} - U_{kj}^{\ast} \| & \text{if $\ell>0$ is odd and $\frac{\ell+1}{2} \mbox{ mod } 2 =j$}, \\
\| V_{kj} - V_{kj}^{\ast} \|  & \text{if $\ell>0$ is even and $\frac{\ell}{2} \mbox{ mod }2 =j$},
\end{cases}
\end{equation*}
and the $2m$-th generalized Fibonacci number
\begin{equation*}
F_\ell  =  \begin{cases}
0 & \text{if $\ell \leq 1$}, \\
c_1 \sqrt{\epsilon} & \text{if $\ell=1$}, \\
c_2(F_{\ell-1}+F_{\ell-2}+\cdots+F_{\ell-2m})  & \text{if if $\ell>1$}.
\end{cases}
\end{equation*}
Then we have
\begin{equation} \label{inexacterror}
\begin{split}
E_\ell
& \leq c_1 \sqrt{\epsilon} + c_2 (E_{\ell-1}+E_{\ell-2}+\cdots+E_{\ell-2m}) \\
& = (\ell-1)c_1 \sqrt{\epsilon} + F_\ell.
\end{split}
\end{equation}
Following the technique of \cite{Kalman1982Generalized} and letting
\[
R =
\begin{bmatrix}
0 & 1 & 0 & \cdots & 0 & 0 \\
0 & 0 & 1 & \cdots & 0 & 0 \\
0 & 0 & 0 & \cdots & 0 & 0 \\
\vdots & \vdots & \vdots &  & \vdots & \vdots \\
0 & 0 & 0 & \cdots & 0 & 1 \\
c_2 & c_2 & c_2 &  & c_2 & c_2
\end{bmatrix},
\]
we have
\[
\begin{bmatrix}
F_\ell &\\
\vdots & \\
F_{\ell+2m-1} &\\
F_{\ell+2m} &
\end{bmatrix}
=
R^\ell
\begin{bmatrix}
0 &\\
\vdots & \\
0 &\\
c_1 \sqrt{\epsilon}  &
\end{bmatrix}.
\]
Provided that there are $2m$ eigenvalues of $R$,
    denoted $r_1,r_2,\dots,r_{2m}$, which can be shown as in
    \citet{Miles1960Generalized} and \citet{Wolfram1998Solving},
    via eigen-decomposition $R=VDV^{-1}$, where
\[
V=
\begin{bmatrix}
1 & 1 & 1 & \cdots & 1 \\
r_1 & r_2 & r_3 &  & r_{2m} \\
r_1^2 & r_2^2 & r_3^2 & \cdots & r_{2m}^2 \\
&  & \vdots &  &  \\
r_1^{2m-1} & r_2^{2m-1} & r_3^{2m-1} & \cdots & r_{2m}^{2m-1}
\end{bmatrix},
\]
we have
\begin{equation} \label{accError}
\begin{split}
F_\ell  = [1,0,\dots,0] V D^\ell V^{-1}
\begin{bmatrix}
0 &\\
\vdots & \\
0 &\\
c_1 \sqrt{\epsilon}  &
\end{bmatrix} 
 = [r_1^\ell,r_2^\ell,\dots,r_{2m}^\ell] V^{-1}
\begin{bmatrix}
0 &\\
\vdots & \\
0 &\\
c_1&
\end{bmatrix} \sqrt{\epsilon} 
= \sum_{i=1}^{2m} r_i^\ell z_i \sqrt{\epsilon},
\end{split}
\end{equation}
where $z_1,\dots,z_{2m}$ satisfy
\[
V
\begin{bmatrix}
z_1 \\
\vdots \\
z_{2m-1} \\
z_{2m}
\end{bmatrix}
=
\begin{bmatrix}
0 \\
\vdots \\
0 \\
c_1
\end{bmatrix}.
\]
Combining \eqref{inexacterror} and \eqref{accError}, we have
\[
E_\ell = \left[ (\ell-1)c_1 + \sum_{i=1}^{2m} r_i^\ell z_i \right] \sqrt{\epsilon}  = O( r^\ell \sqrt{\epsilon}),
\]
where $r$ is the largest eigenvalue of $R$, which completes the proof.

\section{Technical Lemmas}

\begin{lemma} \label{lemma:compute-cov}
	Let $A=\mM \operatorname{diag}(\sigma_1,\dots,\sigma_{d_2}) \mN^\top$ where $\mM = (M_1,\dots,M_{d_2})$, $\mN = (N_1,\dots,N_{d_2})$. If $\mZ \in \R^{d_1\times d2}$ such that $\E{\Vector(\mZ)}=\zeros$ and $\E{\Vector(\mZ)\Vector(\mZ)^\top} = \operatorname{diag}(\theta_{11},\theta_{21},\dots,\theta_{d_1 d_2})$, then we have
	\[
	\E{\mZ \mA \mZ^\top} = \sum_{i=1}^{d_2} \sigma_i \operatorname{diag}( M_i^\top\operatorname{diag}(\Theta_1) N_i, \dots,  M_i^\top\operatorname{diag}(\Theta_{d_2}) N_i),
	\]
	where $\Theta_i=(\theta_{i 1},\dots,\theta_{i d_2})$.
	In particular, $\E{\Vector(\mZ)\Vector(\mZ)^\top} = \mI$ and
	\[
	\E{\mZ \mA \mZ^\top} =  \operatorname{Trace}(\mA) \mI.
	\]
\end{lemma}
\begin{proof}
	For $\mZ^\top =(Z_1, \dots, Z_{d_2})$ and $\Theta_i=(\theta_{i1},\dots,\theta_{i d_2})$, we know
	\[
	[\E{(\mZ M)(\mZ N)^\top}]_{i,j} =
	\begin{cases}
	\E{M^\top Z_i Z_i^\top N} = M^\top\operatorname{diag}(\Theta_i) N,& \text{if } i=j\\
	0,              & \text{otherwise}
	\end{cases}.
	\]
	Then, we have
	\[
	\E{\mZ \mA \mZ^\top} = \sum_{i=1}^{d_2} \sigma_i \E{ \mZ M_i N_i^\top \mZ^\top}
	= \sum_{i=1}^{d_2} \sigma_i \operatorname{diag}( M_i^\top\operatorname{diag}(\Theta_1) N_i, \dots,  M_i^\top\operatorname{diag}(\Theta_{d_2}) N_i).
	\]
\end{proof}

\begin{lemma} \label{lemma:nosiy-PM}
	Let $\mT = \mU \operatorname{diag}(\sigma_1,\sigma_2,\dots, \sigma_k) \mV^\top$ be SVD.
	Let $U, V$ be the largest left and right singular vectors, respectively.
	Then we have
	\begin{equation}
	\tan \theta(U, \mT W+G) \leq \max \left( \epsilon, \max \left( \epsilon, (\sigma_{2}/\sigma_1)^{1/4} \right) \tan \theta(V,W) \right),
	\end{equation}
	provided that
	\begin{align*}
	4 \| U^\top G \| \leq& (\sigma_{1} - \sigma_{2}) \cos \theta(V, W), \\
	4 \| G \| \leq& (\sigma_{1} - \sigma_{2}) \epsilon,
	\end{align*}
	where $\theta(V, W) = \arccos (V^\top W/\|V\|/\|W\|)$.
\end{lemma}
\begin{proof}
	W.L.O.G, we assume $\|W\|=1$.
	Define $\Delta = (\sigma_1-\sigma_2)/4$. From assumptions, we have
	\begin{align*}
	\frac{\|U^\top G \|}{\|V^\top W \|} &= \frac{\|U^\top G\|}{\cos \theta(V, W)} \leq (\sigma_1 - \sigma_2)/4 = \Delta,\\
	\frac{\|\mU_\bot^\top G \|}{\|V^\top W \|} &\leq \frac{ \|G\|}{\cos \theta(V, W)} \leq \epsilon \Delta(1+\tan \theta(V, W)),
	\end{align*}
	where $\mU_\bot^\top \mU_\bot =\mI$ and $\mU_\bot U=\zeros$ and last inequality follows by using that $1/\cos \theta \leq 1+ \tan \theta$.
	Then we have
	\begin{align*}
	\tan \theta(U, \mT W+G) =& \frac{\|\mU_\bot^\top (\mT W+G)\|}{\|U^\top (\mT W+G\|} \leq \frac{\|\mU_\bot^\top \mT W\| +\|\mU_\bot^\top G\|}{\|U^\top \mA W\| - \| U^\top G\|} \\
	\leq& \frac{1}{\|V^\top W\|} \frac{\sigma_2\|\mU_\bot^\top W\| +\|\mU_\bot^\top G)\|}{\sigma_1 - \| U^\top G\| / \|V^\top W\|}\\
	=& \frac{1}{\|V^\top W\|} \frac{\sigma_2}{\sigma_2 + 3 \Delta} + \frac{\epsilon \Delta (1+\tan \theta(V,X))}{\sigma_2 + 3\Delta}\\
	=& \left( 1-\frac{\Delta}{\sigma_2+3\Delta} \right) \frac{\sigma_2+\epsilon\Delta}{\sigma_2 + 2 \Delta} \tan \theta(V,X) + \frac{ \Delta}{\sigma_2 + 3\Delta} \epsilon \\
	\leq& \max \left( \epsilon,\frac{\sigma_2+\epsilon\Delta}{\sigma_2 + 2 \Delta} \tan \theta(V,X)  \right),
	\end{align*}
	where the last inequality follows by the fact that the weighted mean of two terms is less than their maximum.
	Moreover, we have
	\[
	\frac{\sigma_2+\epsilon\Delta}{\sigma_2 + 2 \Delta} \leq \max \left(\frac{\sigma_2}{\sigma_2 + \Delta}, \epsilon \right),
	\]
	because the left hand side is a weighted mean of the components on the right, and
	\[
	\frac{\sigma_2+\Delta}{\sigma_2 + \Delta} \leq \left( \frac{\sigma_2+\Delta}{\sigma_2 + 4 \Delta} \right)^{1/4} \leq \left( \frac{\sigma_2}{\sigma_1} \right)^{1/4},
	\]
	which gives the result.
\end{proof}

\begin{lemma}[Theorem 1.3 (Matrix Hoeffding) in \cite{tropp2012user}] \label{lemma:matrix-hoeffding}
	Consider a finite sequence $\{\mM_k\}$ of independent, random, self-adjoint matrices with dimension $d$, and let $\{\mR_k\}$ be a sequence of fixed self-adjoint matrices. Assume that each random matrix satisfies
	\[
	\E{\mM_k}=\zeros, \text{ and},\  \mM^2_k \leq \mR_k^2 \text{ almost surely.}
	\]
	Then, for all $t\geq0$,
	\[
	\operatorname{Prob}\left\{ \lambda_{\max}\left(\sum_k \mM_k\right)\geq \tau \right\}\leq d\exp{(-\tau^2/8\sigma^2)}, \text{ \mbox{where} } \sigma^2=\left\|\sum_k \mR_k^2\right\|.
	\]
\end{lemma}

\begin{lemma}[Lemma 8 in \cite{Fukumizu2007}] \label{lemma:AB32}
	Suppose $\mA$ and $\mB$ are positive symmetric matrices such that $0 \leq \mA \leq \lambda \mI$ and $0 \leq \mB \leq \lambda \mI$. Then
	\[
	 \|\mA^{3/2} - \mB^{3/2} \| \leq 3\lambda^{1/2} \|\mA-\mB\|.
	\]
\end{lemma}

\begin{lemma}
    \label{lemma:technical}
	Under Assumption \ref{assumption1}, we have, for all $j,k$:
	\begin{enumerate}
		\item $U_{kj}^\top ( \frac{1}{n} \mX_{kj}^\top \mX_{kj}) U_{kj} \in [\sigma_{l,x},\sigma_{u,x}]$.

		\item $\alpha_{kj} \in [\sigma_{u,x}^{-1/2},\sigma_{l,x}^{-1/2}]$.

		\item $1 > \lambda_{k+1,j} - \lambda_{k,j} = [(1-2\lambda_{k,j}) - (1-2\lambda_{k+1,j})]/2 > 0$ and, thus, $\rho_n(\cU_{kj},\cV_{kj})$ converges.

		\item There exists $\sigma_o>0$, independent of $k,j$, such that
		\begin{multline*}
		\tilde{\cL}(\alpha_{kj}, \cU_{k+1,j-1}, \lambda_{k,j}, \beta_{k+1,j-1}, \cV_{k+1,j-1}, \mu_{k+1,j-1})  \\
		-  \tilde{\cL}(\alpha_{k+1,j}, \cU_{k+1,j}, \lambda_{k+1,j}, \beta_{k+1,j-1}, \cV_{k+1,j-1}, \mu_{k+1,j-1}) \\
		\geq \frac{\sigma_0}{2} \left[ (\alpha_{k+1,j}-\alpha_{k+1,j+1})^2 + \|U_{kj}-U_{k+1,j} \|^2+ (\lambda_{k+1,j} - \lambda_{k+1,j+1})^2\right].
		\end{multline*}
		\end{enumerate}
\end{lemma}
	\begin{proof}
		We only prove the case for $m=2$. Extension to an arbitrary $m$ is straightforward.
		The first statement follows from the following two identities
		\[
		U_{kj}^\top ( \frac{1}{n} \mX_{kj}^\top \mX_{kj}) U_{kj} = (U_{k2} \otimes U_{k1} )^\top (\frac{1}{n} \sum_{t=1}^{n} \mbox{vec}(\cX_t)\mbox{vec}(\cX_t)^\top) (U_{k2} \otimes U_{k1} )
		\]
		and
		\[
		(U_{k2} \otimes U_{k1} )^\top (U_{k2} \otimes U_{k1} ) = (U_{k2}^\top U_{k2} \otimes U_{k1}^\top U_{k1}) = \|U_{k2}\|^2 \|U_{k1}\|^2 = 1.
		\]
		The second statement follows from the definition of $\alpha$ and the first statement.

		For statement 3, by Proposition~\ref{prop:hopm-als},
		option I and II of HOPM have the same correlation in each iteration.
		Therefore, since HOPM solves the subprolem where all except
		one component are fixed, the correlation $\rho_n(\cU_{kj}, \cV_{kj})$ is increasing
		at every update and the first statement follows by the assumption.
		Note that $(\alpha_{kj}, \cU_{kj})$ is feasible, i.e.,
		\[
		\frac{\alpha_{kj}^2}{n} U_{kj}^\top \mX_{k,j}^\top \mX_{k,j}U_{kj} = 1 = \frac{\alpha_{k+1,j}^2}{n} U_{k+1,j}^\top \mX_{kj}^\top \mX_{kj}U_{k+1,j},
		\]
		and so
		\begin{equation} \label{eq1}
		\begin{split}
		& \frac{1}{2} \left[  \tilde{\cL}(\alpha_{kj}, \cU_{k+1,j-1}, \lambda_{k,j}, \beta_{k+1,j-1}, \cV_{k+1,j-1}, \mu_{k+1,j-1}) \right.  \\
		& \qquad - \left.  \tilde{\cL}(\alpha_{k+1,j}, \cU_{k+1,j}, \lambda_{k+1,j}, \beta_{k+1,j-1}, \cV_{k+1,j-1}, \mu_{k+1,j-1}) \right]  \\
		& =\frac{1}{2} \left[  \frac{\alpha_{kj} \beta_{k,j-1}}{n} U_{kj}^\top \mX_{kj}^\top \mY_{kj} V_{k,j-1} - \frac{\alpha_{k+1,j} \beta_{k,j-1}}{n} U_{k+1,j}^\top \mX_{kj}^\top \mY_{kj} V_{k,j-1} \right] \\
		& =  \frac{1}{2} \left[ (1-2\lambda_{kj} )- (1 -2\lambda_{k+1,j})\right]   \\
		& \geq  \lambda_{k+1,j}-\lambda_{kj} \\
		& \geq  (\lambda_{k+1,j}-\lambda_{kj})^2,
		\end{split}
		\end{equation}
		where the last inequality holds because $1>\lambda_{k+1,j}-\lambda_{kj} >0$.
		Furthermore, we have
		\begin{equation}
		\begin{split}
		\tilde{\cL}(\alpha_{kj}, &\cU_{k+1,j-1}, \lambda_{k,j}, \beta_{k+1,j-1}, \cV_{k+1,j-1}, \mu_{k+1,j-1}) \\
		& \qquad - \tilde{\cL}(\alpha_{k+1,j}, \cU_{k+1,j}, \lambda_{k+1,j}, \beta_{k+1,j-1}, \cV_{k+1,j-1}, \mu_{k+1,j-1}) \\
		& =  \tilde{\cL}(\alpha_{kj}, \cU_{k+1,j-1}, \lambda_{k+1,j}, \beta_{k+1,j-1}, \cV_{k+1,j-1}, \mu_{k+1,j-1}) \\
		& \qquad - \tilde{\cL}(\alpha_{k+1,j}, \cU_{k+1,j}, \lambda_{k+1,j}, \beta_{k+1,j-1}, \cV_{k+1,j-1}, \mu_{k+1,j-1}) \\
		& = \tilde{\cL}(\alpha_{kj}, \cU_{k+1,j-1}, \lambda_{k+1,j}, \beta_{k+1,j-1}, \cV_{k+1,j-1}, \mu_{k+1,j-1})
		\\
		& \qquad - \tilde{\cL}(\alpha_{kj}, \cU_{k+1,j}, \lambda_{k+1,j}, \beta_{k+1,j-1}, \cV_{k+1,j-1}, \mu_{k+1,j-1}) \\
		& \qquad +  \tilde{\cL}(\alpha_{k,j}, \cU_{k+1,j}, \lambda_{k+1,j}, \beta_{k+1,j-1}, \cV_{k+1,j-1}, \mu_{k+1,j-1}) \\
		& \qquad - \tilde{\cL}(\alpha_{k+1,j}, \cU_{k+1,j}, \lambda_{k+1,j}, \beta_{k+1,j-1}, \cV_{k+1,j-1}, \mu_{k+1,j-1}).
		\end{split}
		\end{equation}
		From the statement 3, we have
		\begin{equation*}
		\begin{split}
		\nabla_{\alpha} \tilde{\cL}(\alpha_{k+1}, \cU_{k+1,j}, \lambda_{k+1,j}, \beta_{k+1,j-1}, \cV_{k+1,j-1}, \mu_{k+1,j-1}) & = 0, \\
		\nabla^2_{\alpha} \tilde{\cL}(\alpha_{k+1}, \cU_{k+1,j}, \lambda_{k+1,j}, \beta_{k+1,j-1}, \cV_{k+1,j-1}, \mu_{k+1,j-1})
		& = (1-2\lambda_{k+1,j}) U_{k+1,j}^\top (\frac{1}{n} \mX_{kj}^\top \mX_{kj}) U_{k+1,j} >0,
		\end{split}
		\end{equation*}
		which implies
		\begin{multline} \label{eq2}
		\tilde{\cL}(\alpha_{k,j}, \cU_{k+1,j}, \lambda_{k+1,j}, \beta_{k+1,j-1}, \cV_{k+1,j-1}, \mu_{k+1,j-1}) \\
		- \tilde{\cL}(\alpha_{k+1,j}, \cU_{k+1,j}, \lambda_{k+1,j}, \beta_{k+1,j-1}, \cV_{k+1,j-1}, \mu_{k+1,j-1}) \\
		\geq (1-2\lambda_{0,0}) \sigma_{l,x} (\alpha_{k,j}-\alpha_{k+1,j})^2.
		\end{multline}
		Also,
		\begin{equation} \label{eq3}
		\begin{split}
		\tilde{\cL}(\alpha_{kj}, &\cU_{k+1,j-1}, \lambda_{k+1,j}, \beta_{k+1,j-1}, \cV_{k+1,j-1}, \mu_{k+1,j-1}) \\
		& \qquad - \tilde{\cL}(\alpha_{kj}, \cU_{k+1,j}, \lambda_{k+1,j}, \beta_{k+1,j-1}, \cV_{k+1,j-1}, \mu_{k+1,j-1}) \\
		& = \alpha_{k,j}^2 (1-2\lambda_{k+1,j}) (U_{k+1,j-1}-U_{k+1,j})^\top (\frac{1}{n} \mX_{kj}^\top \mX_{kj}) (U_{k+1,j-1}-U_{k+1,j}) \\
		& \qquad - \alpha_{kj} \beta_{k,j-1} U_{kj}^\top (\frac{1}{n} \mX_{kj}^\top \mY_{kj}) V_{k,j-1} + \alpha_{kj} \beta_{k,j-1} U_{k+1,j}^\top (\frac{1}{n} \mX_{kj}^\top \mY_{kj}) V_{k+1,j-1} \\
		& \geq \sigma_{u,x}^{-1} (1-2\lambda_{0,0}) (U_{k+1,j-1} - U_{k+1,j})^\top (\frac{1}{n} \mX_{kj}^\top \mX_{kj}) (U_{k+1,j-1}-U_{k+1,j})
		\end{split}
		\end{equation}
		where the last inequality follows by the fact that
		\[
		U_{k+1,j}^\top (\frac{1}{n} \mX_{kj}^\top \mY_{kj}) V_{k,j-1} - U_{kj}^\top (\frac{1}{n} \mX_{kj}^\top \mY_{kj}) V_{k,j-1} >0,
		\]
		which can be shown by the following:
		Let $f(U)=U^\top (\frac{1}{n} \mX_{kj}^\top \mY_{kj}) V_{k,j-1}$ be a linear
		function w.r.t. $U$ with the gradient $(\frac{1}{n} \mX_{kj}^\top \mY_{kj}) V_{k,j-1}$.
		Since $\tilde{U}_{k+1,j} = \mX_{kj}^\dagger \mY_{kj} V_{k,j-1}$,
		$V_{k,j-1}^\top (\frac{1}{n}\mY_{kj}^\top \mX_{kj} ) \mX_{kj}^\dagger \mY_{kj} V_{k,j-1} >0$,
		and $\|U_{kj}\|=1=\|U_{k+1,j}\|$,
		by the property of the project gradient descent on the unit ball, $f((U_{k,j} + \epsilon  {U}_{k+1,j}/ \| U_{k,j} + \epsilon  {U}_{k+1,j} \|)) > f(U_{k,j})$ for all $\epsilon>0$. Letting $\epsilon \rightarrow \infty$, we obtain the desired result.

		Combining \eqref{eq1}, \eqref{eq2}, \eqref{eq3}, Statement 1 and Assumption \ref{assumption1}, we complete the lemma.
	\end{proof}

The following lemma establishes the fact that the updating variables never go to zero.

\begin{lemma} \label{lemma:boundedsolution}
	Under Assumption~\ref{assumption1}, for all $k,j$, we have
	\[
	    \| U_{kj} \| > \sigma_{u,x}^{-1} \sigma_{l,x}^{1/2} \sigma_{l,y}^{1/2}
	    \quad \text{and}\quad
	    \| V_{kj} \| > \sigma_{u,y}^{-1} \sigma_{l,y}^{1/2} \sigma_{l,x}^{1/2}.
	\]
\end{lemma}
\begin{proof}
	We only show the first statement. It is easy to see that
	\begin{align*}
    	\| U_{kj} \| & = \| (\mX_{kj}^\top \mX_{kj})^{-1} \mX_{kj} \mY_{kj} V_{k,j-1} \| \\
	& \geq \sigma_{u,x} \sigma_{l,x}^{1/2} \sigma_{l,y}^{1/2},
	\end{align*}
	where the last inequality holds because, for any unit vector $U$, we have
	\[
	U^\top \mX_{kj}^\top \mX_{kj} U = (U_{km} \otimes \cdots \otimes U \otimes \cdots \otimes U_{k1})^\top \frac{1}{n} \sum_{t=1}^{n} \mbox{vec}(\cX_t)\mbox{vec}(\cX_t)^\top (U_{km} \otimes \cdots \otimes U \otimes \cdots \otimes U_{k1}).
	\]
	This complete the proof.
\end{proof}

\section{More Analysis for Air Pollution Data in Taiwan} \label{appendix:air-pollution}

The wind and rain conditions differ dramatically between summer and winter in Taiwan.
In the summer, typhoons and afternoon thunderstorms are common and
wind is from the south-east (Pacific Ocean). They reduce the air pollutant concentrations.
In contrast, it is dry with strong seasonal wind from the north-west (Mainland of China) in the winter,
leading to higher measurements of pollutant concentrations in winter months.
To illustrate these meteorological effects, we divide the data by summer and winter.
Specifically, January to March and October to December are winter and April to September
are summer. Tables~\ref{corrM}, \ref{loadingStationsM}, and \ref{loadingPollutantM}
summarize the results of the analysis. This separation reveals more information.
The coefficient of Meinong, for instance, is large in Table~\ref{loadingStationsG}
due to its location. However, in Table~\ref{loadingStationsM}, it is significantly
different between winter and summer in north stations versus south stations.
This is understandable because the north side of Meinong station is blocked by
mountains that reduce the wind effect in winter (see Appendix~\ref{TWmap}.)

Table \ref{corrM} shows that, as expected, the correlations between regions are
higher during the winter. It is also interesting to see the differences in loadings of stations
between winter and summer in Table~\ref{loadingStationsM}. For instance,
consider South vs East, the loadings of the Eastern stations change sign between winter and
summer. This is likely to be caused by the change in wind direction. Finally, loadings of
PM10 are smaller than those of other pollutants in Table \ref{loadingPollutantM},
indicating that  PM10 behaves differently from the others.

\begin{table}[t]
	\centering
	\begin{tabularx}{\textwidth}{YYYY}
		\hline
		{} & North vs South & South vs East & North vs East \\
		\hline
		Winter &  0.940 &  0.904 &  0.930 \\
		Summer &  0.889 &  0.821 &  0.914 \\
		\hline
	\end{tabularx}
	\caption{The correlations between regions by season of Taiwan air pollutants.} \label{corrM}
\end{table}

\begin{table}[p]
	\centering
	\begin{tabularx}{\textwidth}{Y|YYYY|YYYY}
		\hline
		&&North&&&&South&&\\
		\hline
		N vs S &    Guting &   Tucheng &   Taoyuan &   Hsinchu &     Erlin &   Xinying &  Xiaogang &   Meinong \\
		\hline
		Winter &  0.423 &   0.060 &   0.078 &   0.901 &  -0.925 &  -0.344 &    0.161 &  -0.024 \\
		Summer &  0.400 &  -0.277 &  -0.189 &  -0.853 &   0.834 &   0.435 &    0.206 &   0.270 \\
		\hline
		&&North&&&&South&&\\
		\hline
		N vs E &    Guting &   Tucheng &   Taoyuan &   Hsinchu &     Yilan &  Dongshan &   Hualien &   Taitung \\
		\hline
		Winter &   0.599 &   0.723 &   0.339 &   0.060 &  -0.819 &   -0.273 &  -0.500 &  -0.064 \\
		Summer &  -0.419 &  -0.225 &  -0.783 &  -0.402 &  -0.843 &   -0.195 &  -0.465 &  -0.188 \\
		\hline
		&&South&&&&East&&\\
		\hline
		S vs E &     Erlin &   Xinying &  Xiaogang &   Meinong &     Yilan &  Dongshan &   Hualien &   Taitung \\
		\hline
		Winter &  0.961 &  -0.061 &    0.116 &   0.244 &  -0.462 &   -0.143 &  -0.751 &  -0.449 \\
		Summer &  0.354 &   0.798 &    0.300 &   0.386 &   0.733 &   -0.363 &   0.332 &   0.469 \\
		\hline
	\end{tabularx}
	\caption{The loading of monitoring stations of the first CCA by season.} \label{loadingStationsM}
\end{table}

\begin{table}[p]
	\centering
	\begin{tabularx}{\textwidth}{YYYYYYYY}
		\hline
		{} &       SO2 &        CO &        O3 &      PM10 &       NOx &        NO &       NO2 \\
		\hline
		N(Winter) &   0.122 &   0.672 &  -0.176 &  -0.084 &  0.298 &  -0.083 &  -0.633 \\
		S(Winter) &   0.228 &   0.805 &   0.357 &   0.071 &  0.292 &  -0.242 &   0.152 \\
		N(Summer) &  -0.540 &  -0.344 &   0.532 &   0.060 &  0.263 &  -0.480 &   0.066 \\
		S(Summer) &  -0.000 &   0.136 &  -0.360 &   0.003 &  0.232 &   0.740 &  -0.499 \\
		\hline
		N(Winter) &   0.477 &   0.330 &  -0.247 &  -0.078 &  0.397 &  -0.430 &  -0.504 \\
		E(Winter) &  -0.558 &  -0.228 &   0.333 &   0.150 &  0.417 &  -0.484 &  -0.307 \\
		N(Summer) &  -0.111 &   0.113 &   0.074 &   0.188 &  0.450 &  -0.283 &  -0.807 \\
		E(Summer) &   0.272 &  -0.371 &   0.127 &   0.209 &  0.099 &   0.473 &  -0.704 \\
		\hline
		S(Winter) &   0.404 &   0.027 &  -0.263 &  -0.019 &  -0.627 &   0.392 &   0.469 \\
		E(Winter) &  -0.304 &  -0.523 &   0.171 &   0.077 &   0.280 &  -0.712 &  -0.115 \\
		S(Summer) &  -0.608 &  -0.103 &   0.174 &   0.009 &   0.455 &  -0.300 &  -0.541 \\
		E(Summer) &  -0.111 &   0.024 &   0.269 &   0.046 &   0.262 &   0.619 &  -0.679 \\
		\hline
	\end{tabularx}
	\caption{The loadings of  pollutants of the first CCA by season.} \label{loadingPollutantM}
\end{table}

\section{Electricity Demands in Adelaide} \label{Appendix:experiments}

In this example we investigate the relationship between electricity demands in
Adelaide, Australia, and temperatures measured at Kent Town from Sunday to Saturday
between 7/6/1997 and 3/31/2007. The demands and temperatures are measured every half-hour
and we represent the data as two 508 by 48 by 7 tensors. To remove the diurnal patterns in the data,
we remove time-wise median from the measurements. Diurnal patterns are common in such
data as they are affected by human activities and daily weather.
We also consider data for day time (10 am to 3 pm) and
evening time (6 pm to 11 pm) only to provide a deeper analysis.

We apply the TCCA to the median-adjusted half-hourly electricity demands and temperatures.
Tables~\ref{corrDT}, \ref{loadingDay}, \ref{loadingTime} and \ref{loadingTimeDN}
summarize the results.
From Table~\ref{corrDT}, the maximum correlation between
electricity demand and temperature is 0.973, which is high.
This is not surprising as unusual temperatures
(large deviations from median) tend to require use of heating or air conditioning.
On the other hand, when we focus on data on day time or evening time, the maximum
correlations become smaller, but remain substantial. Table~\ref{loadingDay}
shows that (a), as expected, the loadings are all positive and similar in size for each day when
all data are used, but (b) the loadings for the evening change sign
between weekday and weekend. This indicates that people in Adelaide, Australia,
behave differently  in the evenings between weekday and weekend.

Table~\ref{loadingTime} shows that (a) the loadings in the afternoon (from 2 pm to 4 pm) and
evening (from 6 pm to 8 pm) tend to be higher and positive, (b) the loadings during the sleeping
time (from 11pm to 3am) are small and negative. This behavior is also understandable
because people use less electricity while they are sleeping and the temperature tends
to be cooler in the evening.


\begin{table}[t]
	\centering
	\begin{tabularx}{\textwidth}{YYY}
		\hline
		All &        Day(10am-3pm) &         Evening(6pm-11pm) \\
		\hline
		0.973 &  0.885 &  0.714 \\
		\hline
	\end{tabularx}
	\caption{Correlations between electricity demands and temperatures. Half-hourly
	data from Adelaide and Kent Town, Australia.} \label{corrDT}
\end{table}

\begin{table}[p]
	\centering
	\begin{tabularx}{\textwidth}{cXXXXXXX}
		\hline
		&    Mon. &   Tue. &  Wed. &  Thu. &    Fri. &  Sat. &    Sun. \\
		\hline
		Whole day &  0.240 &  0.413 &   0.385 &  0.436 &  0.420 &  0.441 &  0.244 \\
		Day (10am-3pm) &  0.375 &  0.437 &   0.440 &  0.432 &  0.396 &  0.299 &  0.198 \\
		Evening (6pm-11pm) &  0.195 &  0.325 &   0.528 &  0.160 & -0.323 & -0.584 & -0.322 \\
		\hline
	\end{tabularx}
	\caption{Loadings of days for the electricity demands and temperature data.} \label{loadingDay}
\end{table}

\begin{table}[p]
	\centering
	\begin{tabularx}{\textwidth}{YYYYYYYY}
		\hline
		0  &          &        1  &          &        2  &          &     3     &          \\
		-0.080 & -0.076 & -0.104 & -0.104 & -0.061 & -0.021 &  0.021 &  0.067 \\
		\hline
		4  &         &       5 &         &        6 &         &        7 &         \\
		0.099 &  0.121 &  0.12 &  0.075 & -0.002 & -0.086 & -0.129 & -0.144 \\
		\hline
		8 &         &        9 &         &        10 &         &        11 &         \\
		-0.155 & -0.156 & -0.151 & -0.125 & -0.064 &  0.002 &  0.066 &  0.116 \\
		\hline
		12 &        &        13 &         &        14 &        &        15 &         \\
		0.145 &  0.174 &  0.179 &  0.177 &  0.206 &  0.238 &  0.262 &  0.284 \\
		\hline
		16 &         &        17 &        &        18 &         &        19 &         \\
		0.281 &  0.227 &  0.121 & -0.054 & -0.231 & -0.237 & -0.214 & -0.206 \\
		\hline
		20 &         &        21 &         &        22 &         &        23 &        \\
		-0.135 & -0.071 & -0.058 & -0.069 & -0.080 & -0.055 & -0.004 & -0.062 \\
		\hline
	\end{tabularx}
	\caption{Loadings of half-hour interval for the electricity demand and temperature data, time is shown by hours.} \label{loadingTime}
\end{table}

\begin{table}[p]
	\centering
	\begin{tabularx}{\textwidth}{YYYYYYYYYY}
		\hline
		10 &         &        11 &         &        12 &   & 13 & &       14 &\\

		-0.611 & -0.385 & -0.165 & -0.084 &  0.012 & 0.198 &  0.281 &  0.264 &  0.321 &  0.391 \\
		\hline
		18 &         &        19 &        &        20 &  &  21 &         &        22 & \\
		0.676 &  0.523 &  0.384 &  0.275 &  0.167 & 0.109 &  0.066 &  0.033 &  0.025 &  0.029 \\
		\hline
	\end{tabularx}
	\caption{Loadings of daytime and nighttime, time is shown by hours.} \label{loadingTimeDN}
\end{table}

\newpage

\section{Taiwan Map} \label{TWmap}

\begin{figure}[h]
	\centering
	\begin{subfigure}[b]{0.9\textwidth}
		\includegraphics[width=\textwidth]{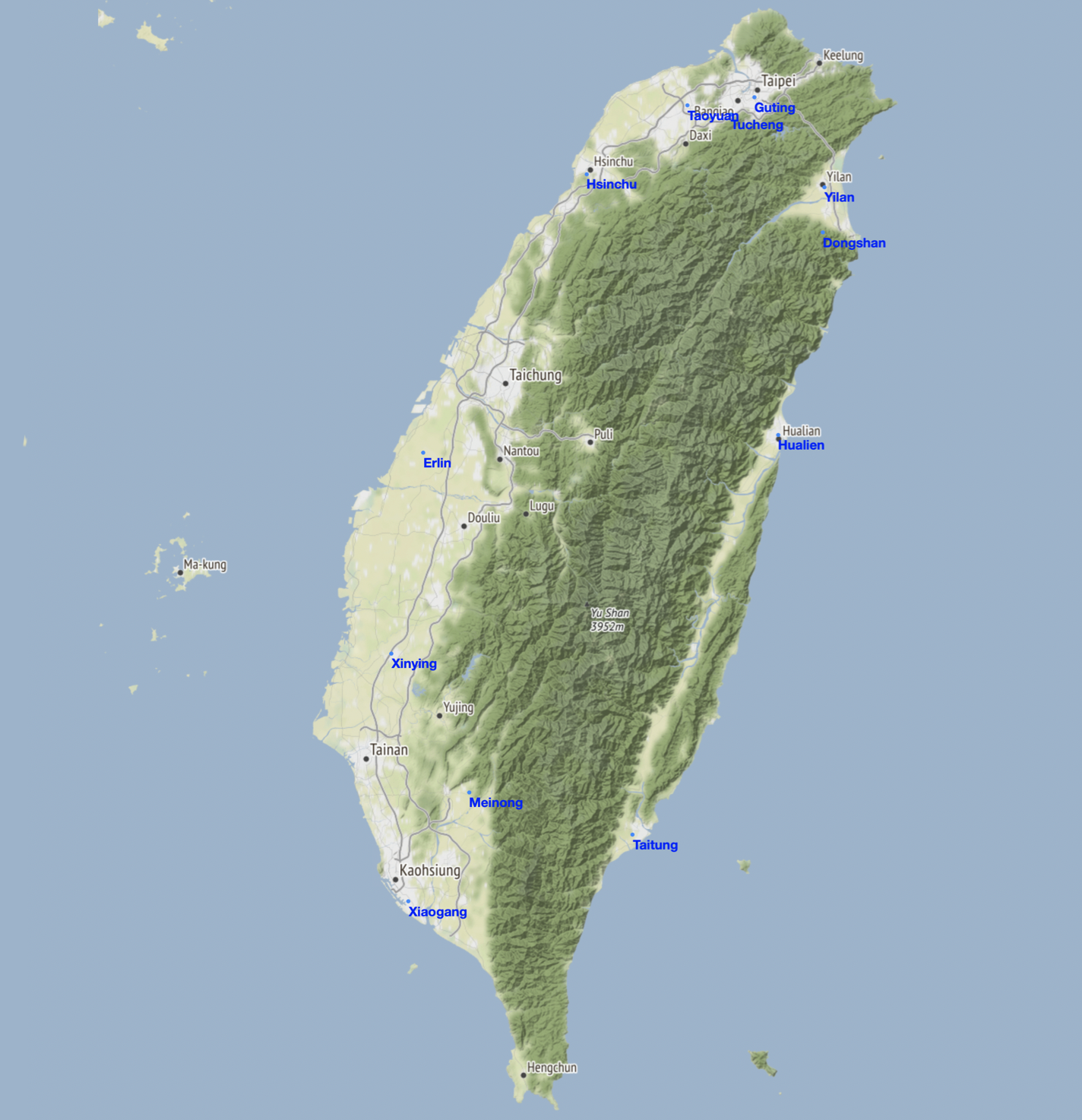}
	\end{subfigure}
	\caption{Air Pollution Monitoring Stations in Taiwan}
\end{figure}

%
%

\end{document}